\documentclass[pdflatex,sn-mathphys]{sn-jnl}

\usepackage{amsmath}
\usepackage{amsthm}
\usepackage{xcolor}
\usepackage[font=footnotesize]{caption} 
\usepackage{subcaption}
\usepackage[normalem]{ulem}
\usepackage{enumitem}
\usepackage{comment}
\usepackage{microtype}
\usepackage{graphbox}
\usepackage{natbib}

\newcommand{\Ncal}{\mathcal{N}}
\newcommand{\Dcal}{\mathcal{D}}
\usepackage{color}

\usepackage{setspace}
\usepackage{ulem}

\newlength{\tempdima}
\newcommand{\rowname}[1]
{\rotatebox{90}{\makebox[\tempdima][c]{\textbf{#1}}}}

\newcommand{\mycaption}[1]
{\refstepcounter{subfigure}\textbf{(\thesubfigure) }{\ignorespaces #1}}

\providecommand{\keywords}[1]
{
  \small	
  \textbf{\textit{Keywords---}} #1
}

\usepackage{graphicx}
\usepackage{float}

\newtheorem{lemma}{Lemma}

\jyear{2021}%

\theoremstyle{thmstyleone}%
\theoremstyle{thmstyletwo}%

\theoremstyle{thmstylethree}%

\begin{document}

\title{A framework for benchmarking uncertainty in deep regression}

\author*[1]{\fnm{Franko} \sur{Schm\"ahling}}\email{franko.schmaehling@ptb.de}

\author[1]{\fnm{J\"org} \sur{Martin}}\email{joerg.martin@ptb.de}
\equalcont{These authors contributed equally to this work.}

\author[1]{\fnm{Clemens} \sur{Elster}}\email{clemens.elster@ptb.de}
\equalcont{These authors contributed equally to this work.}

\affil*[1]{\orgname{Physikal.-Techn. Bundesanstalt Braunschweig \& Berlin}, \orgaddress{\country{Germany}}}


\abstract{We propose a  framework for the assessment of uncertainty quantification in deep regression. The framework is based on regression problems where the regression function is a linear combination of nonlinear functions.
Basically, any level of complexity can be realized through the choice of the nonlinear functions and the dimensionality of their domain. 
Results of an uncertainty quantification for deep regression are compared against those obtained by a statistical reference method.
The reference method utilizes knowledge of the underlying nonlinear functions and is based on a Bayesian linear regression using a reference prior. Reliability of uncertainty quantification is assessed in terms of coverage probabilities, and accuracy through the size of calculated uncertainties.
We illustrate the proposed framework by applying it to current approaches for uncertainty quantification in deep regression. 
The flexibility, together with the availability of a reference solution, makes the  framework suitable for defining benchmark sets for uncertainty quantification.}


\maketitle

\section{Introduction}

Methods from deep learning have made tremendous progress over recent years and are meanwhile routinely applied in many different areas. Examples comprise the diagnosis of cancer \cite{KOOI2017303}, autonomous driving \cite{huang2020autonomous} or language processing \cite{Otter_NLP_DL,Deng_DL_NLP}. Physical and engineering sciences also benefit increasingly from deep learning and current applications range from optical form measurements \cite{HoffmannL2020} over image quality assessment in mammography \cite{kretz2020mammography,Kretz2020} or medical imaging in general \cite{LUNDERVOLD2019102} to applications in weather forecasts and climate modeling \cite{Schultz_DL_weather,Watson-Parris_DL_weatherclimate}.

Especially for safety relevant applications, such as medical diagnosis, it is crucial to be able to quantify the uncertainty associated with  a prediction. Uncertainty quantification can help to understand the behavior of a trained neural net and, in particular, foster confidence in its predictions. This is especially true for deep regression, where a single point estimate of a sought function without any information regarding its accuracy can be largely meaningless. 
	
Bayesian networks \cite{Bishop2006, Gal2016,Depeweg2018} lend themselves to an uncertainty quantification in a natural way. All network parameters are modeled as random variables, and their posterior distribution can be used to calculate the epistemic part of the uncertainty. However, the shear number of network parameters prohibits the application of established computation techniques such as Markov chain Monte Carlo \cite{RobertCasella2005}, which have been successfully applied in numerous Bayesian inferences on low-dimensional models.
	
Approximate Bayesian techniques, and in particular Variational Bayesian Inference  \cite{Bishop2006}, have therefore been suggested for deep learning. At the same time, these methods are closely related to classical network regularization techniques \cite{Hinton2012,Srivastava2014,pmlr-v28-wang13a}.

Approaches based on ensemble learning \cite{Lakshminarayanan2017} comprise another branch of current uncertainty quantification methods. These techniques scale well and are easy to implement. Furthermore, they have been shown to perform very successfully compared to other methods for uncertainty quantification \cite{Caldeira2020,Gustafsson2020,Ovadia2019,Scalia2020}.
	
In accordance with the growing efforts on the development of methods for uncertainty quantification in deep learning, the amount of work on comparing different uncertainty quantification approaches increases as well \cite{Yao2019, Levi2019, Gustafsson2020, Kendall2017, Ovadia2019, Caldeira2020,Scalia2020}. Different ways of assessing the performance of uncertainty quantification methods have been proposed, ranging from explorative comparisons of the performance on test data for simple regression problems with known ground truth to various metrics \cite{Gustafsson2020,Levi2019}.\\
We propose a novel framework for the assessment of uncertainty quantification in deep regression. The framework is based on regression problems of the kind

\begin{equation}\label{eq0.1}
y_i \sim \Ncal \left( G(x_i)^T \gamma, \sigma^2 \right)\,,
\end{equation}
that are linear in the parameter $\gamma=(\gamma_1,\gamma_2, \ldots)^T$, and where $G(x)=(G_1(x), G_2(x),\ldots)^T$ is a vector built of nonlinear functions $G_1, \,G_2,\,\ldots~$. Deep regression intends to infer the regression function $G(x)^T \gamma$, and it does not make use of the specific structure of \eqref{eq0.1}. The level of complexity is controlled through the choice of the nonlinearity in the function $G(x)$ as well as by the dimensionality of its input $x$. Results of an uncertainty quantification for deep regression of problems of the kind in \eqref{eq0.1} can be compared to those obtained by a statistical reference method. The reference method utilizes knowledge of $G$ and the linearity with respect to the unknown parameters $\gamma$. It consists of a Bayesian inference using the improper prior $\pi(\gamma)\propto 1$, which coincides with the reference prior for the statistical model \eqref{eq0.1} \cite{berger1992development}. This prior is probability matching \cite{Kass&Wasserman1996} so that the proposed reference uncertainty quantification can also be recommended from the perspective of classical statistics. The statistical reference method provides an anchor solution against which different methods of uncertainty quantification for deep regression can be compared.

Performance of uncertainty quantification is here assessed in terms of reliability and accuracy. The former is quantified through coverage probabilities of interval estimates for the regression function at individual inputs $x$, and the latter by the size of the associated uncertainty. We concentrate on the assessment of the epistemic uncertainty which is relevant when inference of the regression function is the goal of the task. The performance metrics have been chosen on grounds of their simplicity and typical use in statistics. Nevertheless, other performance metrics, cf. for instance \cite{Gustafsson2020,Levi2019}, could be used as well in conjunction with the proposed framework.
	
The assessment of methods for uncertainty quantification in deep regression is illustrated by applying it to several current approaches and for several choices of the function $G(x)$. 

The paper is structured as follows. In section \ref{sec:theory} we present the proposed structure of test problems along with their treatment by a Bayesian reference inference.  In section \ref{sec:experiments} we illustrate the usage of the framework from section \ref{sec:theory} for various choices of the underlying nonlinearity and study the performance of several deep learning methods for uncertainty quantification. Finally, some conclusions are given in section \ref{sec:ConclusionOutlook}.

\section{Creating reference models for uncertainty quantification}
\label{sec:theory}
\subsection{A primer on (deep) regression and uncertainty } 
\label{subsec:primer_deep_regression}
A typical, normal regression for training data $\Dcal=\{(x_1,y_y),\ldots,(x_N,y_N)\}$ assumes a model of the sort
\begin{align}
    \label{eq:generic_normal_regression}
    y_i\vert x_i  \sim \Ncal(f_\theta(x_i), \sigma^2),
\end{align}
where $\Ncal$ denotes a normal distribution, $f_\theta$ is some function, parametrized by $\theta$, and $\sigma^2$ denotes the variance of underlying noise. For simplicity we will, in this work, only consider \textit{one-dimensional outputs} $y_i$,
and we will also assume knowledge about $\sigma^2$. The input $x_i$ and the parameter $\theta$ will be of arbitrary dimensions.
Assuming independent observations, \eqref{eq:generic_normal_regression} gives rise to the statistical model 
$p(\mathcal{D}\vert\theta)=\prod_{i=1}^N \Ncal(y_i\vert f_\theta(x_i),\sigma^2)$
and inferring $\theta$ could be done, for instance, by maximizing $p(\mathcal{D}\vert\theta)$ with respect to $\theta$ (maximum likelihood estimate - MLE) or by a  Bayesian approach, that is by computing the posterior 
\begin{align}
    \label{eq:BayesTheorem}
   \pi(\theta\vert \mathcal{D}) \propto \pi(\theta) \cdot p(\mathcal{D}\vert \theta)\,,
\end{align}
where $\pi(\theta)$ denotes the prior for $\theta$. The uncertainty associated with $\theta$ might, in the Bayesian case, be taken as the variation arising from the distribution \eqref{eq:BayesTheorem} while, in the frequentist case, often the variation of the MLE under the data distribution is used.

In the case of deep regression, $f_\theta$ is given by a neural network and the parameter $\theta$ contains all parameters 
of the network, such as weights and biases, and is therefore typically of a pretty high dimension ($10^6$ or more). In such cases there is usually no hope of sampling from $\pi(\theta\vert \Dcal)$ or computing its normalization constant. Instead, one has to rely on approximations. In deep learning, many methods \cite{Graves2011,Kingma2015,Gal2017,Blundell2015,martin2021errorsinvariables} rely on the statistical tool of variational inference, that is to approximate $\pi(\theta\vert \Dcal)$ via a feasible distribution $q_\phi(\theta)$, parametrized by a hyperparameter $\phi$, by minimizing the Kullback-Leibler divergence with respect to $\phi$. Depending on the class of variational distributions $q_\phi$ that is chosen, a variety of approximations arise without a clear gold standard. Besides approaches based on variational inference, there are other approaches to quantify the uncertainty about $\theta$  \cite{Pearce2018,Lakshminarayanan2017,Barber98ensemblelearning,Bishop94mixturedensity}. The approach of deep ensembles \cite{Lakshminarayanan2017}, introduced as a sort of frequentist approach, can also be viewed as a Bayesian approach \cite{he2020bayesian, Hoffmann2021}.

\subsection{A generic model for benchmarking uncertainty} \label{subsec:genericmodel4benchmark}
While for neural networks it is non-trivial to obtain an uncertainty for $\theta$ in  \eqref{eq:generic_normal_regression}, for problems that are linear in a lower dimensional parameter this is standard. In this section we will exploit this fact to propose a generic scheme to create benchmarking test problems of arbitrary complexity. 
Given a generally nonlinear function $G(x)$ with $p$-dimensional output and a $p$-dimensional parameter $\gamma$, we can define the following model
\begin{align}
    \label{eq:genericmodelbenchmark}
    y_i \sim \Ncal(G(x_i)^T\gamma, \sigma^2)\,,
\end{align}
where we assume $\sigma^2$ to be known. We will use \eqref{eq:genericmodelbenchmark} in three different ways:
\begin{enumerate}
    \item To generate \textit{training data} $\mathcal{D}$: fix input data $\mathbf{x}=(x_1,\ldots,x_N)$ and generate (independent) $\mathbf{y} = (y_1,\ldots,y_N)$ according to \eqref{eq:genericmodelbenchmark}. 
    Store both, $\mathbf{x}$ and $\mathbf{y}$, in $\mathcal{D}=\{ (x_i,y_i)\, \vert \, i=1,\ldots,N\}$.
    \item To generate \textit{test data} with a known ground truth: we fix test inputs  $\textbf{x}^\ast=(x_1^\ast, \ldots, x_{N^\ast}^\ast)$. The values $G(x_i^\ast)\gamma$ will be taken as the (unnoisy) ground truth to judge the quality of predictions and uncertainties of the studied models.
    \item To create an \textit{anchor model}: the test input $\textbf{x}^\ast$ from 2. is evaluated under a Bayesian inference, based on $\mathcal{D}$ from 1., and the ``true model'' \eqref{eq:genericmodelbenchmark}. This will be described in the following and is summarized by the distribution \eqref{eq:BayesSolutionPredictionSolution} below. The anchor model serves as a reference point to judge the performance of uncertainty methods for other (neural network based) models of the more general shape \eqref{eq:generic_normal_regression}.
\end{enumerate}
The map $G(x)$ in \eqref{eq:genericmodelbenchmark} allows to tune the nonlinearity and complexity underlying the training data $\mathcal{D}$. The linearity of \eqref{eq:genericmodelbenchmark} with respect to $\gamma$, on the other hand, allows for an explicit Bayesian inference in the anchor model (assuming knowledge about $G(x)$). We here propose to use the (improper) reference prior $\pi(\gamma) \propto 1$ for the statistical model \eqref{eq:genericmodelbenchmark}, see e.g. \cite{berger1992development}. Taking this prior has, in our case, the convenient effect of exact probability matching, 
and hence the coverage probability under repeated sampling of $\mathbf{y}$ of credible intervals derived from the Bayesian posterior equals their credibility level. This ensures that credible and confidence intervals coincide. The according posterior distribution can be specified explicitly \cite{gelman2013bayesian} as
\begin{align}
    \label{eq:BayesSolutionPosterior}
    \pi(\gamma\vert \mathcal{D}) \propto \prod_{i=1}^{n} e^{-\frac{1}{2\sigma^{2}}(y_{i}-G(x_{i})^T\gamma)^{2}}\propto 
    e^{-\frac{1}{2\sigma^2}(\gamma-\hat{\gamma})^{T}V^{-1}(\gamma-\hat{\gamma})}
    \propto \Ncal( \gamma\vert \hat{\gamma}, \sigma^2 V)    \,,
\end{align}
where $\hat{\gamma} =  V G(\mathbf{x})\mathbf{y},\, V=(G(\mathbf{x}) G(\mathbf{x})^T)^{-1}$ and where the evaluation $G(\mathbf{x})=$ $(G(x_1)$ $,\ldots,G(x_N))$ $\in \mathbb{R}^{p\times N}$ should be read element-wise.  For the test points $\textbf{x}^\ast$, 
the distribution, conditional on $\mathcal{D}$, for
the prediction $G(\textbf{x}^\ast)^T\gamma$ is then given by
\begin{align}
    \label{eq:BayesSolutionPredictionSolution}
    \pi(G(\textbf{x}^\ast)^T \gamma \vert \mathcal{D}) =\mathcal{N}(G(\textbf{x}^\ast)^T\hat{\gamma}, \sigma^2 G(\textbf{x}^\ast)^T V G(\textbf{x}^\ast)) \,.
\end{align}

We will refer to \eqref{eq:BayesSolutionPredictionSolution} as the \textit{anchor model} in the remainder of this work and use it for benchmarking the performance of other methods on $\mathcal{D}$. For the evaluation we will treat, for the anchor model and all other methods, $\sigma^2$ as known and use the same value that was taken to generate $\mathcal{D}$. 
The distribution \eqref{eq:BayesSolutionPredictionSolution} models the epistemic uncertainty for the test data points $\mathbf{x}^\ast$, which is, as we explained in the introduction, the only uncertainty we study in this work. To consider the aleatoric uncertainty, the posterior predictive distribution $\pi(\mathbf{y}^\ast\vert \mathcal{D}, \mathbf{x}^\ast)=\Ncal(\mathbf{y}^\ast\vert G(\mathbf{x}^\ast)^T\hat{\gamma}, \sigma^2 G(\mathbf{x}^\ast)^T V G(\mathbf{x}^\ast)+ \sigma^2 \mathbf{I})$ could be used instead.

Due to the employed probability matching prior the anchor model
expresses both, the Bayesian and the frequentist view,
and can thus be taken as reference point for approaches from both
worlds. From a frequentist point of view, the anchor model
\eqref{eq:BayesSolutionPredictionSolution} is unbiased and achieves
the Cramer-Rao bound. This makes the model ``optimal'' in the sense of
the following lemma.
\begin{lemma}
    \label{lem:optimality}
   Consider an $x^\ast$ from the domain of $G$ and an estimator $T(\mathbf{y};x^\ast)$ for $G(x^\ast)^T\gamma$, based on $\mathbf{y} \sim \Ncal(G(\mathbf{x})\gamma, \sigma^2 \mathbf{I})$. Consider further some non-negative function $u(\mathbf{y}; x^\ast)$. 
    Then, at least one of the following two statements is true
   \begin{enumerate}
       \item The estimator $T(\cdot\,;x^\ast)$ is biased.
       \item Whenever, for some $\mathbf{y}'$, $u(\mathbf{y}';x^\ast)$ is smaller than the  
   uncertainty of the anchor model \eqref{eq:BayesSolutionPredictionSolution}, that is \begin{align}
        \label{eq:LemmaAssumption}
        u(\mathbf{y}'; x^\ast)^2< \sigma^2 G(x^\ast)^TVG(x^\ast)\,, 
    \end{align}
    then $u(\mathbf{y}'; x^\ast)$ is also smaller than the root mean squared error of $T(\cdot\,;x^\ast)$ under repeated sampling, i.e. $u(\mathbf{y}'; x^\ast)^2<\mathbb{E}_{\mathbf{y}\sim \Ncal(G(\mathbf{x})^T\gamma, \sigma^2\mathbf{I})}\left[(T(\mathbf{y}; x^\ast)-G(x^\ast)^T\gamma)^2\right]$.
   \end{enumerate}
\end{lemma}
\begin{proof}
    Assume that 1 is not true, so that $T(\cdot\,;x^\ast)$ is unbiased. Note first, that, in this case, for any $\mathbf{y}'$ that satisfies
    \eqref{eq:LemmaAssumption} we do have
    \begin{align}\label{eq:lemma_proof_case}
         u(\mathbf{y}'; x^\ast)^2 < \mathrm{Var}_{\mathbf{y}\sim \Ncal(G(\mathbf{x})^T\gamma, \sigma^2\mathbf{I})}(T(\mathbf{y};x^\ast)) \,.
    \end{align}
     Indeed, otherwise we could apply the Cramer-Rao inequality (in the form of \cite[Theorem 3.3]{Shao2003}) to the unbiased estimator $T(\cdot\,;x^\ast)$ and would get a contradiction to \eqref{eq:LemmaAssumption}:
     \begin{align*}
    u(\mathbf{y}';x^\ast)^2\geq \mathrm{Var}_{\mathbf{y}\sim \Ncal(G(\mathbf{x})^T\gamma, \sigma^2\mathbf{I})}(T(\mathbf{y};x^\ast))\geq \sigma^2 G(x^\ast)^TVG(x^\ast)\,.
    \end{align*}
     From \eqref{eq:lemma_proof_case} we arrive at 2, since
     the mean squared error is bounded from below by the variance under repeated sampling
    \begin{align*}
    u(\mathbf{y}';x^\ast)^2 &<\mathrm{Var}_{\mathbf{y}\sim 
\Ncal(G(\mathbf{x})^T\gamma, \sigma^2 \mathbf{I})}(T(\mathbf{y};x^\ast)) \\ &\leq \mathbb{E}_{\mathbf{y}\sim 
\Ncal(G(\mathbf{x})^T\gamma, \sigma^2\mathbf{I})}\left[(T(\mathbf{y};x^\ast)-G(x^\ast)^T\gamma)^2\right]\,.
    \end{align*}
\end{proof}
For our purposes, the statement of lemma \ref{lem:optimality} can be summarized as follows: suppose for some test data point $x^\ast$ and training data $\mathbf{y}$ we have a prediction $T(\mathbf{y};x^\ast)$ of a neural network, for instance the average of the outputs under an approximation to the posterior \eqref{eq:BayesTheorem}, and an uncertainty $u(\mathbf{y}; x^\ast)$, for instance the corresponding standard deviation. 
If the uncertainty is smaller than the one from the anchor model, lemma \ref{lem:optimality} tells us that the prediction is biased or/and the uncertainty is ``too small'', at least compared to the root mean squared error.

\section{Experiments}
\label{sec:experiments}

In this section we study various choices for the benchmarking framework from section \ref{sec:theory}. Namely, we analyze the performance of a couple of state of the art methods for uncertainty quantification in deep learning under test scenarios created with \eqref{eq:genericmodelbenchmark} and compare them with each other and with the anchor model from \eqref{eq:BayesSolutionPredictionSolution}. Methods under test are variational Bayes methods such as Bernoulli-Dropout \cite{Gal2016} (abbreviated "BD" in the following), Concrete Dropout (CD) \cite{Gal2017}, Variational Dropout (VD) \cite{Kingma2015} as well as methods based on Deep Ensembles \cite{Lakshminarayanan2017}. For the latter we used three modes of training: ``standard'' (Ens), with adversarial examples included in training (EnsAdvA) and trained via bootstrapping the training data (EnsBS). In both, Concrete and Variational Dropout, the dropout rate is inherently optimized during the training process. As described in section \ref{sec:theory} the anchor model \eqref{eq:BayesSolutionPredictionSolution} will be included as a reference model. As the anchor model is based on Bayesian linear regression we will use the abbreviation BLR.

Based on the framework from section \ref{sec:theory} we can test these methods on various numerical experiments, arising from different choices of the nonlinear functions $G(x)$ in \eqref{eq:genericmodelbenchmark}. In this work we will consider the following three experiments.$~$\\

\begin{enumerate}[label=\textbf{E.\arabic*}]
    \item \label{E1} As a first choice we will consider a linear combination of $p=4$ sinusoidal functions, with different phases and varying frequencies: 
      \begin{equation}
        \label{eq:sinusoidal_example}
        G(x_{i}) = (\sin(2\pi f_{1} \cdot x_{i} + \rho_{1}), \ldots , \sin(2\pi f_{4} \cdot x_{i} + \rho_{4}))^T\,,
    \end{equation}
    where the $x_i$ are univariate and where the $\gamma \in \mathbb{R}^4$ in \eqref{eq:sinusoidal_example} are randomly chosen from $[0,1]$ (uniform with support in $[0,1]$). All four frequencies $f_{1},\ldots,f_{4}$ are equally spaced between $[0.9, 1.1]\cdot f_{\mathrm{main}}$ and phases $\rho$ are equidistantly chosen between $[0,2\pi]$, where we will consider various choices for $f_{\mathrm{main}}=1,\,2,\,3,\,\ldots$. The used $\sigma$ for this experiment is $0.75$.
    
    \item \label{E2} As a second example, we will use the Styblinsky-Tang-function to scale the dimensionality in an arbitrary manner, i.e. for $d=1,2\ldots,$ we define
    \begin{equation}\label{eq:StyblinkskiTang}
        G(x_{i}) = (x_{i,1},\,x_{i,1}^2,\,x_{i,1}^4,\,x_{i,2},\,x_{i,2}^2,\,x_{i,2}^4,\,\ldots,\,x_{i,d},\, x_{i,d}^2,\,x_{i,d}^4)^T,
    \end{equation}
    with $x_{i,j}$ denoting the $j$-th component of $x_i\in\mathbb{R}^{d}$ and with 
    \begin{equation}\nonumber
    \gamma=(2.5,\,-8,\,0.5,\,2.5,\,-8,\,0.5,\,\ldots,\,2.5,\,-8,\,0.5)^{T}\in \mathbb{R}^{3d}.    
    \end{equation}
     The used $\sigma$ for this experiment is $3$.
    \item \label{E3} As a third example, we will choose a simple polynomial mapping, so that $G^{T}\gamma$ represents a linear combination of linear, quadratic and mixed terms 
    \begin{equation}
        \label{eq:simplepoly2D}
        G(x_{i}) = (1,\,x_{i,1},\,x_{i,2},\,x_{i,1}\cdot x_{i,2},\,x_{i,1}^{2},\,x_{i,2}^{2})^T
    \end{equation}
    with randomly chosen coefficients $\gamma \in \mathbb{R}^7$ (uniform with support in $[0,1]$) in \eqref{eq:simplepoly2D} and $x_{i}\in\mathbb{R}^{2}$. 
    The used $\sigma$ for this experiment is $0.5$.
    
\end{enumerate}
$~$\\
As stated in the introduction of this work we evaluate the performance of the different approaches for uncertainty quantification using common statistical metrics. Namely, we will base the evaluation of the experiments on the following three objectives.$~$\\

\begin{enumerate}
\item \textit{Prediction} for an input: for the anchor model (BLR) and Bayesian deep learning methods (BD, VD, CD) this is the mean of the output of the regression function or the neural network, given the input, when the according parameter is drawn from the posterior distribution (or its approximation). For ensemble based methods (Ens, EnsAdvA, EnsBS) this is the average of the outputs of the single ensemble members for the given input.

We will also compute the \textit{deviation} as the absolute value of the difference between the prediction and the unnoisy ground truth, as defined in section \ref{subsec:genericmodel4benchmark}.
\item \textit{Uncertainty} for an input: for the anchor model (BLR) and Bayesian deep learning methods (BD, VD, CD) this is the standard deviation of the output of the regression function or the neural network, given the input, when the according parameter is drawn from the posterior distribution (or its approximation). For ensemble based methods (Ens, EnsAdvA, EnsBS) this is the standard deviation of the outputs of the single ensemble members for the given input.
\item \textit{Coverage} for an input: draw repeatedly training data under the corresponding model \eqref{eq:genericmodelbenchmark}, for neural network based methods retrain the involved networks, and check whether the deviation, as in point 1, is smaller than 1.96 times the computed uncertainty, as in point 2 (where the 1.96 is motivated from a normal approximation). The coverage denotes the percentage of cases where this succeeded.
For the anchor model we expect a coverage of $95\%$.
\end{enumerate}
$~$\\
\textit{Details on the implementation}: For all experiments below, we used the same network architecture of three hidden, fully connected layers with $128$, $64$, $32$ neurons, a single output neuron and leaky ReLU activations (with negative slope 0.01) between the layers. For the methods based on dropout, dropout layers were added to the next-to-last layer. All methods and networks, together with the benchmarking framework, are implemented by using MATLAB\textsuperscript{\textregistered} and the MATLAB\textsuperscript{\textregistered} ``Deep Learning Toolbox'' \cite{MATLAB:2020b}. All networks were initialized via the Xavier initialisation. For ensemble based methods, we used $120$ ensemble members. In experiments with the sinusoidal generic function \ref{E1} we have $|\text{train set}|=50$ elements in the training set (random distributed in training range) and $|\text{test set}|=10^3$ elements in the test set (chosen equidistantly over the test range). For the experiment \ref{E2} $|\text{train set}|$ scaled with the input dimensionality $d$ as $|\text{train set}|=100\cdot9^{d-1},~d=1,2,3,\ldots$. For the simple 2D polynomial mapping (\ref{E3}) we have $450$ inputs in the train set. The ranges for the test set were $[-6,6]$ for the experiment \ref{E1}, and $[-5,5]$ in every dimension of the input domain for experiment \ref{E2} and \ref{E3}. The range(s) for the train set are $[-4,4]$ (in every input domain) for all above listed experiments.
For all methods, we scaled the number of epochs used for training as $300 \cdot c$, where $c$
was chosen equal to $f_{\mathrm{main}}$ for \ref{E1} and equal to the dimension of the input domain for \ref{E2} and \ref{E3}.
The mini batch size was chosen equal to $|\text{train set}|$, except for ensemble learning with bootstrap (EnsBS), where we took $|\text{train set}|/2$ to run the bootstrapping in a pull with back scenario during one epoch. We used an Adam optimizer \cite{Adam_Kingma2015} with learning rate $10^{-2}$, an $L^2$ regularization of $1/|\text{train set}|$ and momentum parameters $\beta_1 =0.9,\,\beta_2=0.999$ for all experiments.

\subsection{Results} \label{sec:results}
As a first application of the benchmarking framework, we perform the experiment \ref{E1} with different frequencies $f_{\mathrm{main}}$. For every choice of $f_{\mathrm{main}}$ we define a train and a test set as described in section \ref{sec:theory}. As stated above, the input range of the test set is broader, so that we are able to observe the behavior of the methods under test for out-of-distribution data. For the evaluation, we will distinguish between the out-of-distribution range and the range used in the train set (in-distribution). Two examples are illustrated in fig. \ref{fig:exampleOne} and \ref{fig:exampleTwo}. Both figs. show the predictions (solid line) and associated uncertainties times 1.96 (shaded) provided by the anchor model (BLR), Bernoulli Dropout (BD), Variational Dropout (VD) and ensemble with bootstrap (EnsBS), for a ``complexity level'' of $f_{\mathrm{main}}=2$ (see fig. \ref{fig:exampleOne}) and $f_{\mathrm{main}}=5$ (see fig. \ref{fig:exampleTwo}). The ground truth is illustrated as a dashed red line. The ground truth is completely covered by the predictions and associated uncertainties (times 1.96) of the BLR \eqref{eq:BayesSolutionPredictionSolution}. In the in-distribution range nearly all methods cover the ground truth by their associated uncertainties. In the out-of-distribution range the uncertainty grows for the deep learning methods, but is often to small to cover the ground truth.

\begin{figure}[b!]
     \centering
     \begin{subfigure}[b]{0.24\textwidth}
         \centering
         \includegraphics[width=\textwidth]{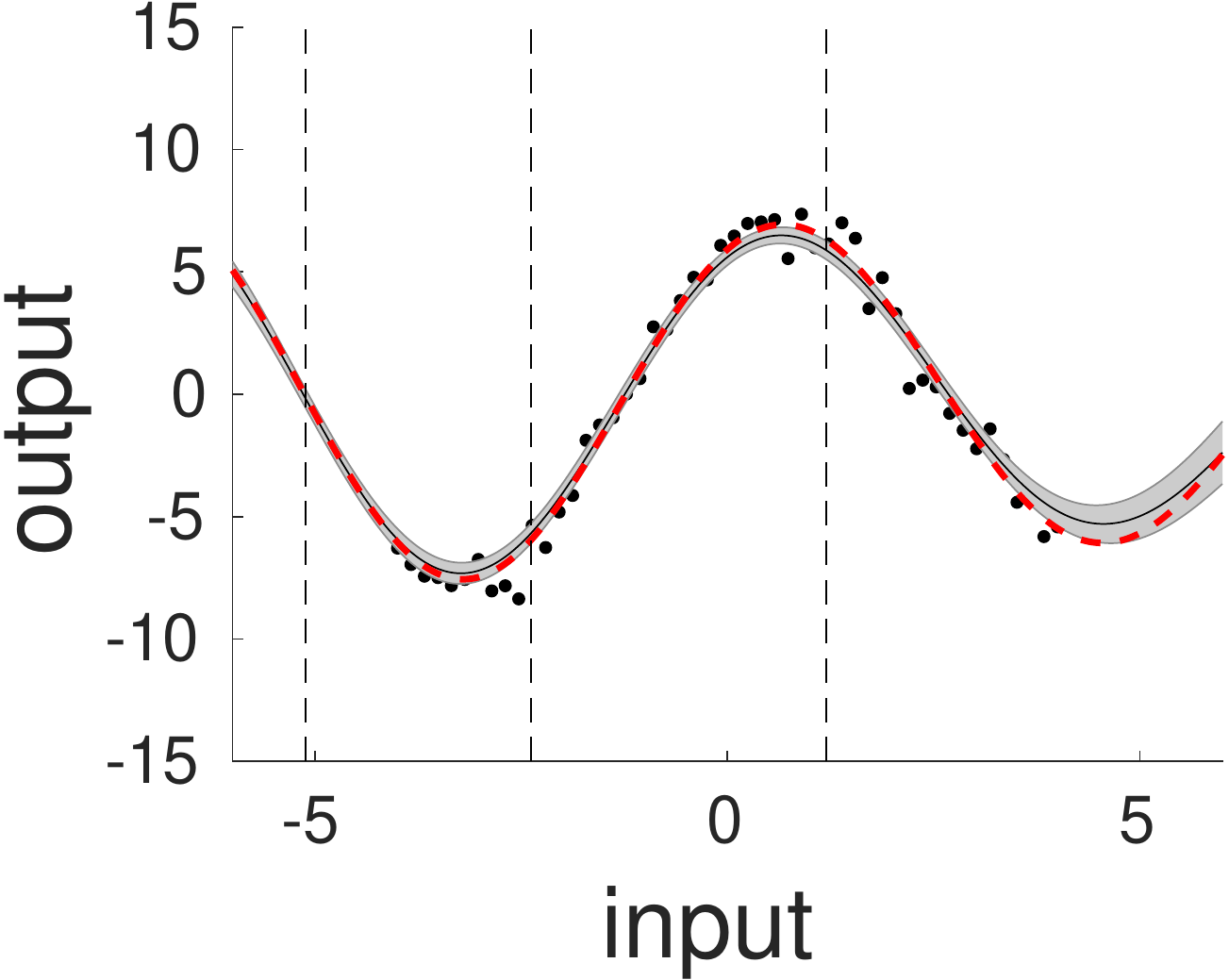}
         \caption{BLR}
     \end{subfigure}
     \hfill
     \begin{subfigure}[b]{0.24\textwidth}
         \centering
         \includegraphics[width=\textwidth]{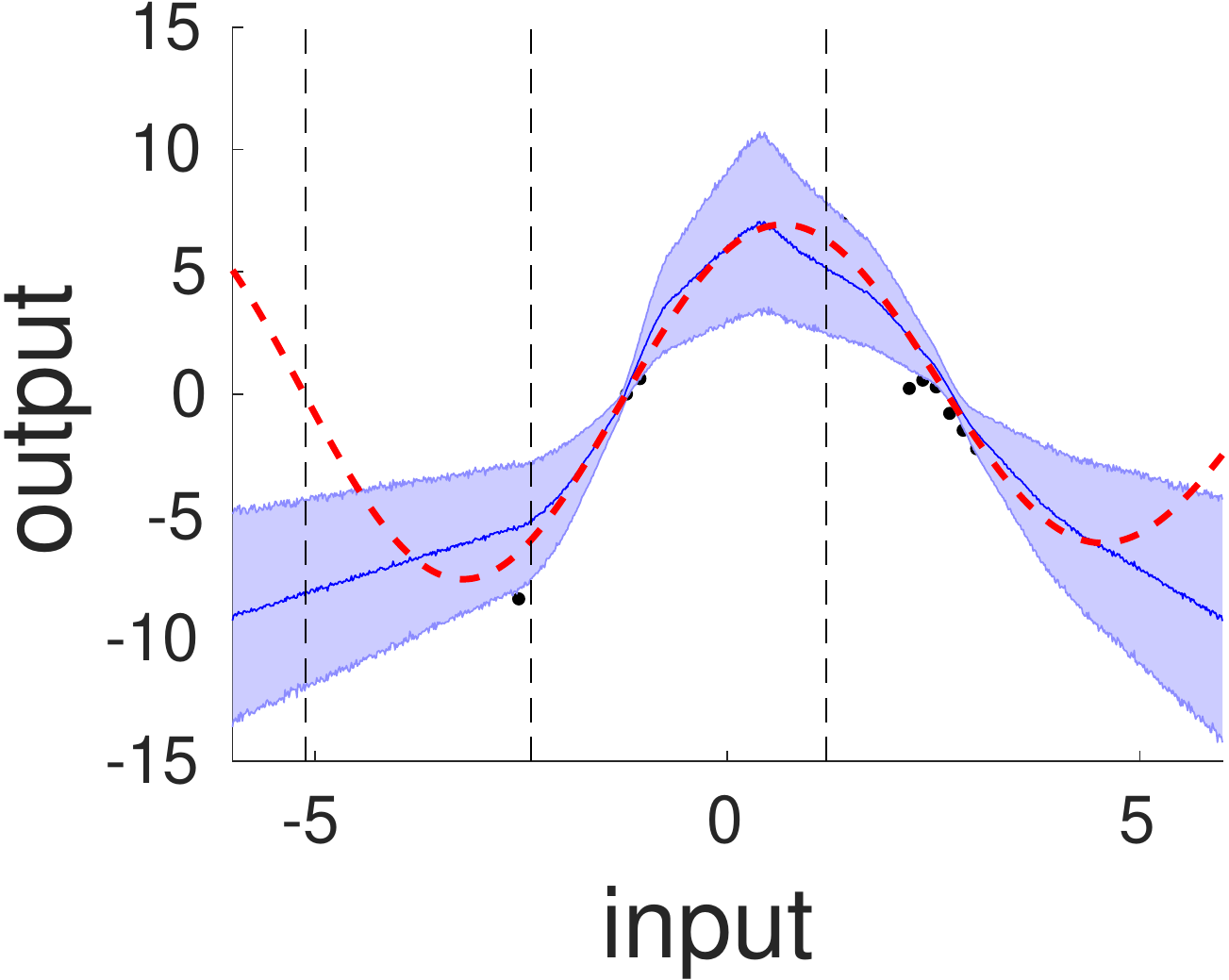}
         \caption{BD}
     \end{subfigure}
     \hfill
     \begin{subfigure}[b]{0.24\textwidth}
         \centering
         \includegraphics[width=\textwidth]{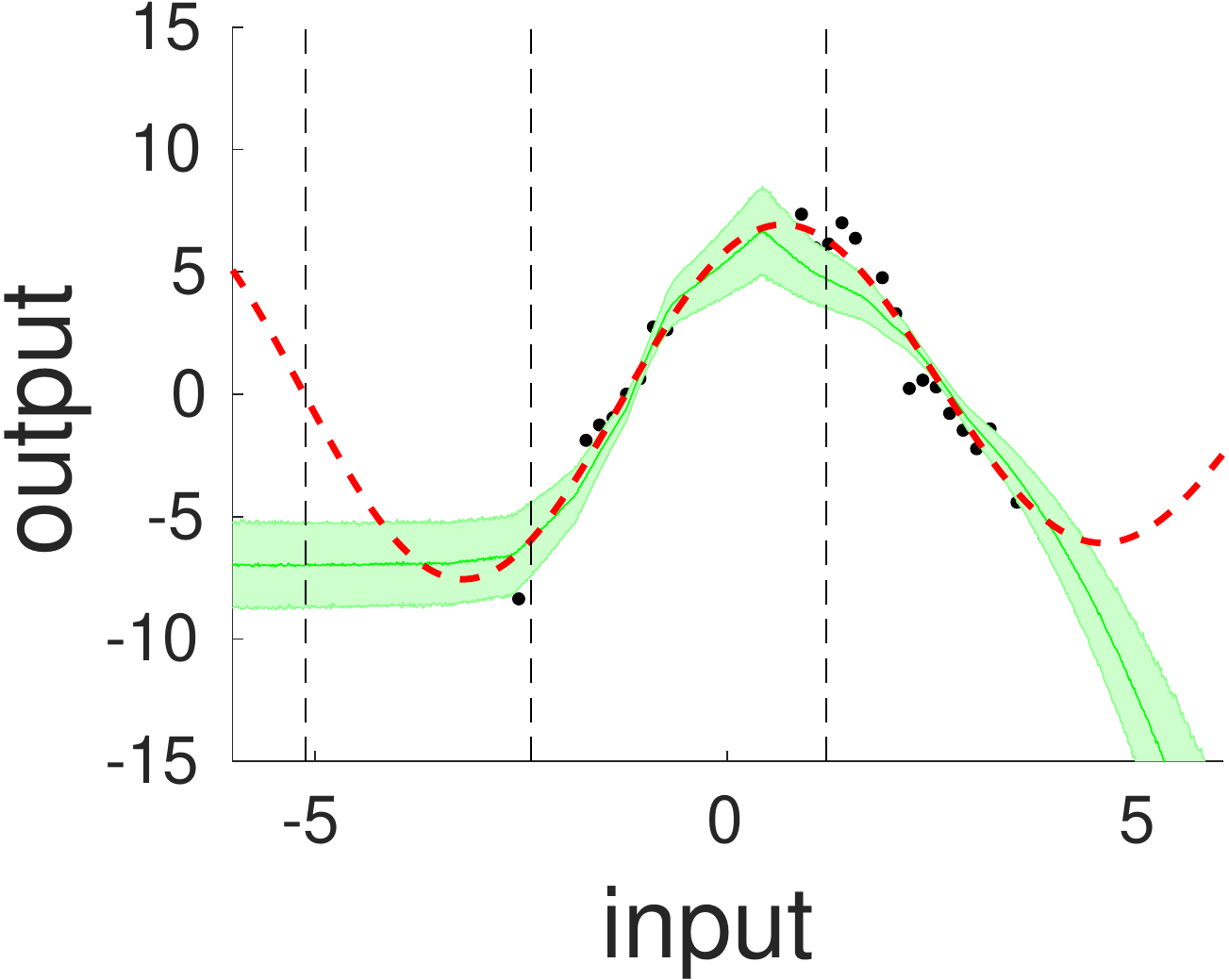}
         \caption{VD}
     \end{subfigure}
     \hfill
     \begin{subfigure}[b]{0.24\textwidth}
         \centering
         \includegraphics[width=\textwidth]{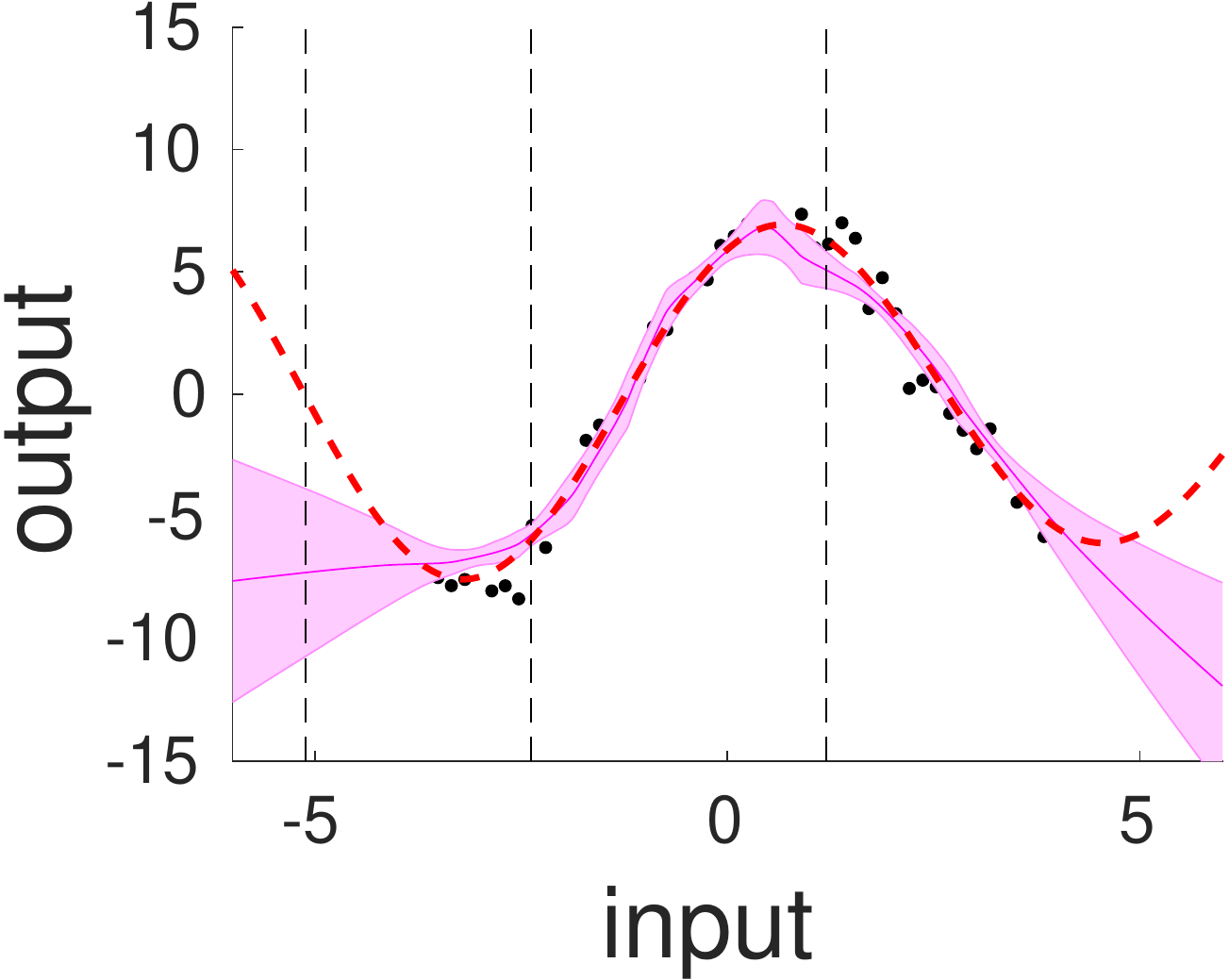}
         \caption{EnsBS}
     \end{subfigure}
        \caption{Experiment \ref{E1} with $f_{\mathrm{main}}=1$. Predictions (solid lines) and calculated uncertainties times 1.96 (shaded areas) of the anchor model BLR (subplot (a)) and different methods in deep learning (subplots (b)-(d)), together with the used train set (black dots) and ground truth (red dashed line). The abbreviations for the methods are as in the beginning of section \ref{sec:results}.} 
        \label{fig:exampleOne}
\end{figure}

\begin{figure}[b!]
     \centering
     \begin{subfigure}[b]{0.24\textwidth}
         \centering
         \includegraphics[width=\textwidth]{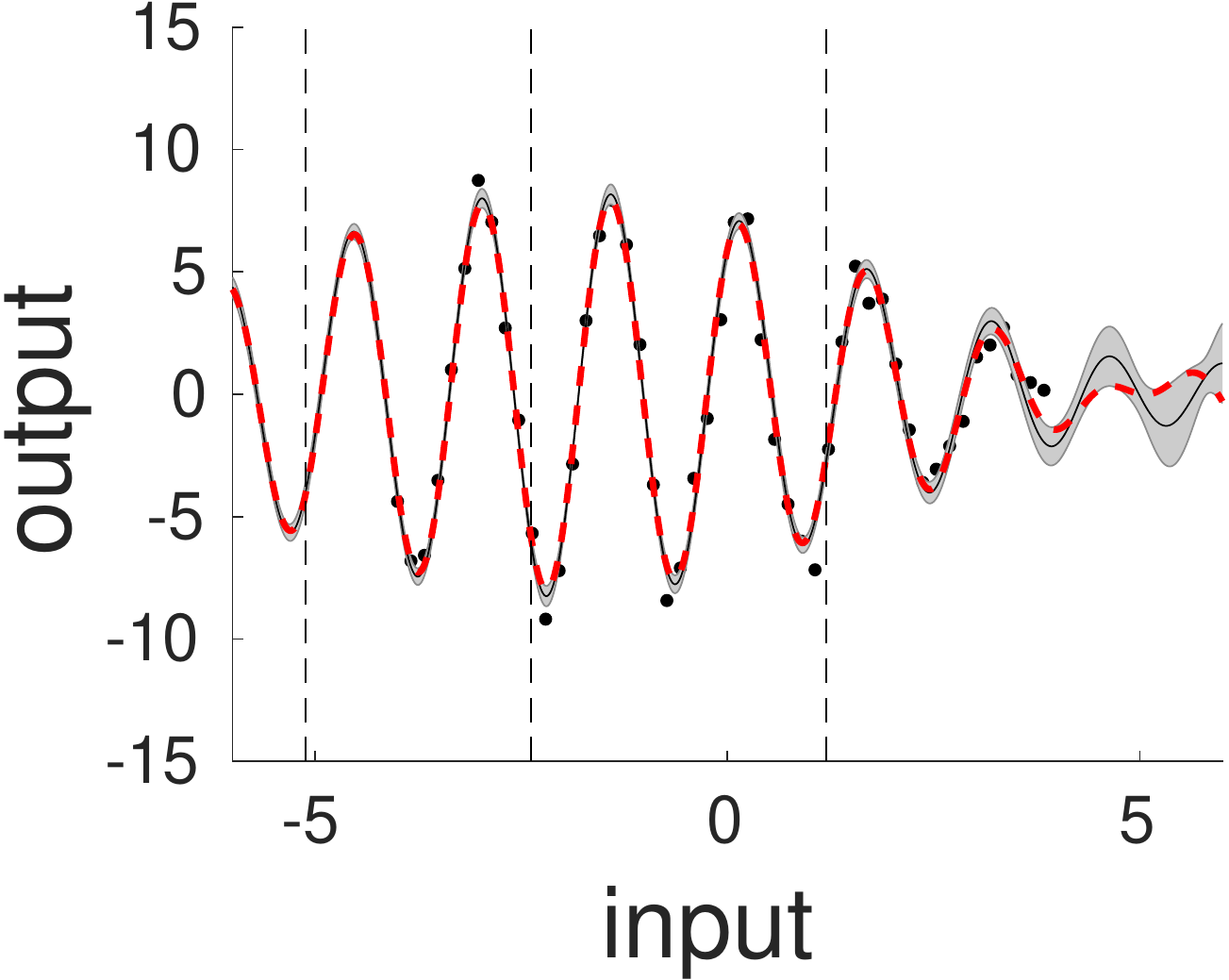}
         \caption{BLR}
     \end{subfigure}
     \hfill
     \begin{subfigure}[b]{0.24\textwidth}
         \centering
         \includegraphics[width=\textwidth]{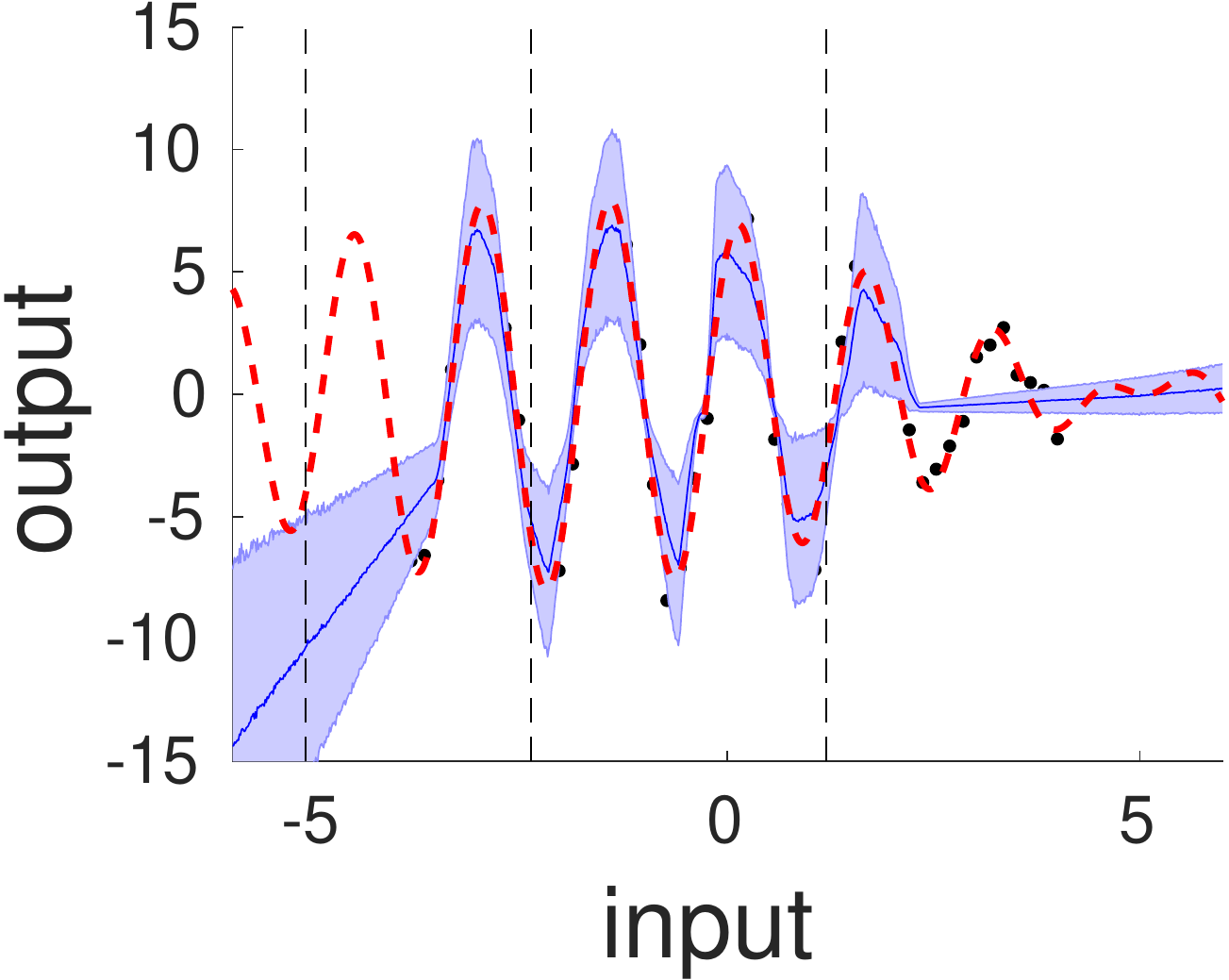}
         \caption{BD}
     \end{subfigure}
     \hfill
     \begin{subfigure}[b]{0.24\textwidth}
         \centering
         \includegraphics[width=\textwidth]{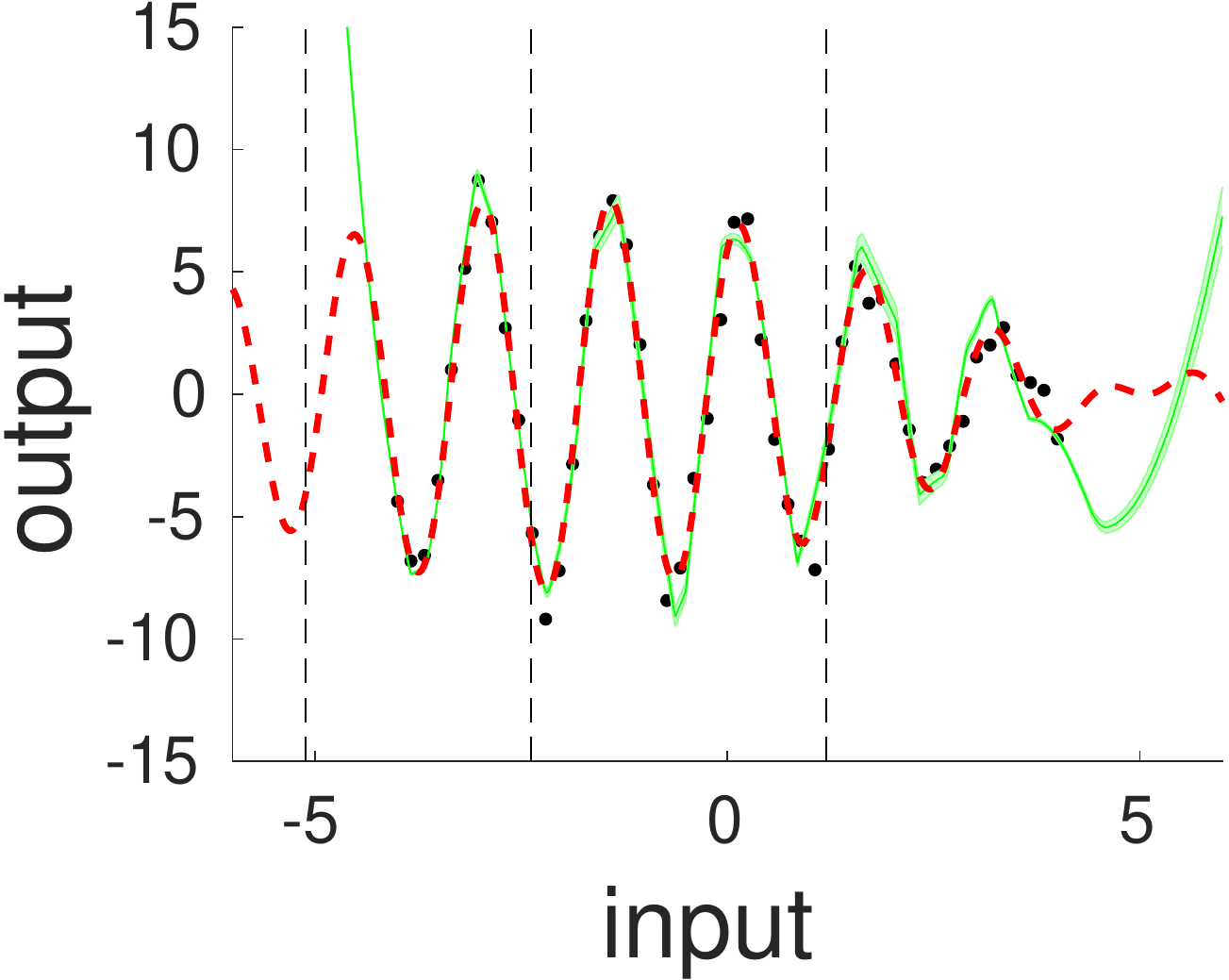}
         \caption{VD}
     \end{subfigure}
     \hfill
     \begin{subfigure}[b]{0.24\textwidth}
         \centering
         \includegraphics[width=\textwidth]{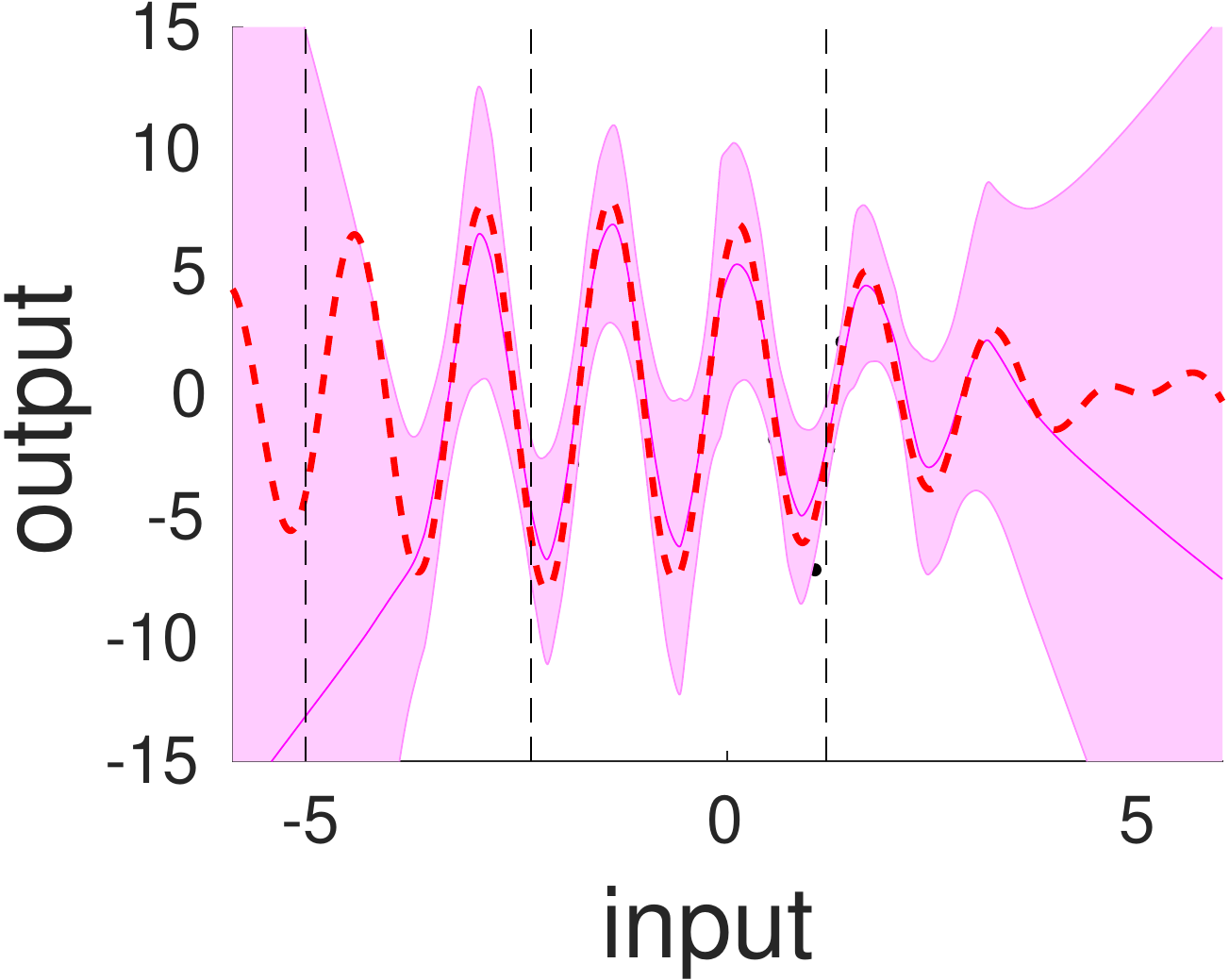}
         \caption{EnsBS}
     \end{subfigure}
        \caption{Experiment \ref{E1} with $f_{\mathrm{main}}=5$. Predictions (solid lines) and calculated uncertainties times 1.96 (shaded areas) of the anchor model BLR (subplot (a)) and different methods in deep learning (subplots (b)-(d)), together with the used train set (black dots) and ground truth (red dashed line). The abbreviations for the methods are as in the beginning of section \ref{sec:results}.}
        \label{fig:exampleTwo}
\end{figure}

Fig. \ref{fig:_3planes_of_dice_iod_ood} summarizes the deviation between the ground truth and the prediction, the uncertainty and the coverage, as specified above, for two in-distribution inputs ($x=-2.38$ and $x=1.2$) and one out-of-distribution input ($x=-5.11$) over varying complexity, i.e increasing frequency $f_{\mathrm{main}}$.
We repeated the experiment, for every $f_{\mathrm{main}}$, $k=50$ times with the same test set but with different realization of the added noise of the train set and re-trained the neural networks in each repetition.

The two left columns show the average of the deviation from the ground truth and the uncertainty together with their associated standard error for the different choices of input $x$. For the coverage (third column from the right) the size of the error bars were generated by using the fact that a sum of $k$ Bernoulli distributed numbers (ground truth in or outside the tube spanned by the $1.96$ times the standard deviation around the estimate) are Binomial distributed with variance $p(1-p)/k$, where p is the success probability, for which the calculated coverage over $k$ repetitions was taken.

\begin{figure}[t!]
     \centering
     \begin{subfigure}[b]{0.3\textwidth}
         \centering
         \includegraphics[width=\textwidth]{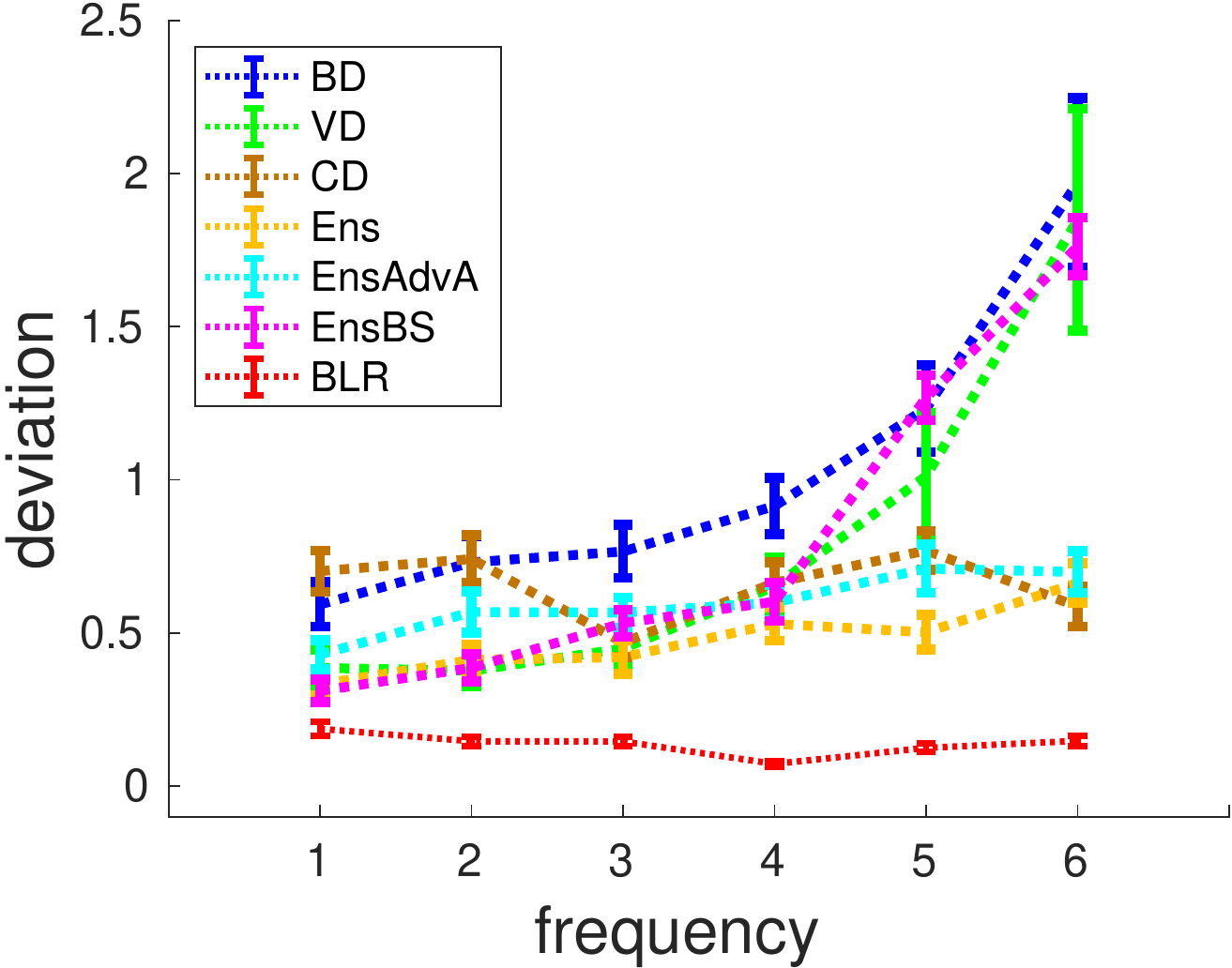}
         \caption{$x=-2.38$}
     \end{subfigure}
     \hfill
     \begin{subfigure}[b]{0.3\textwidth}
         \centering
         \includegraphics[width=\textwidth]{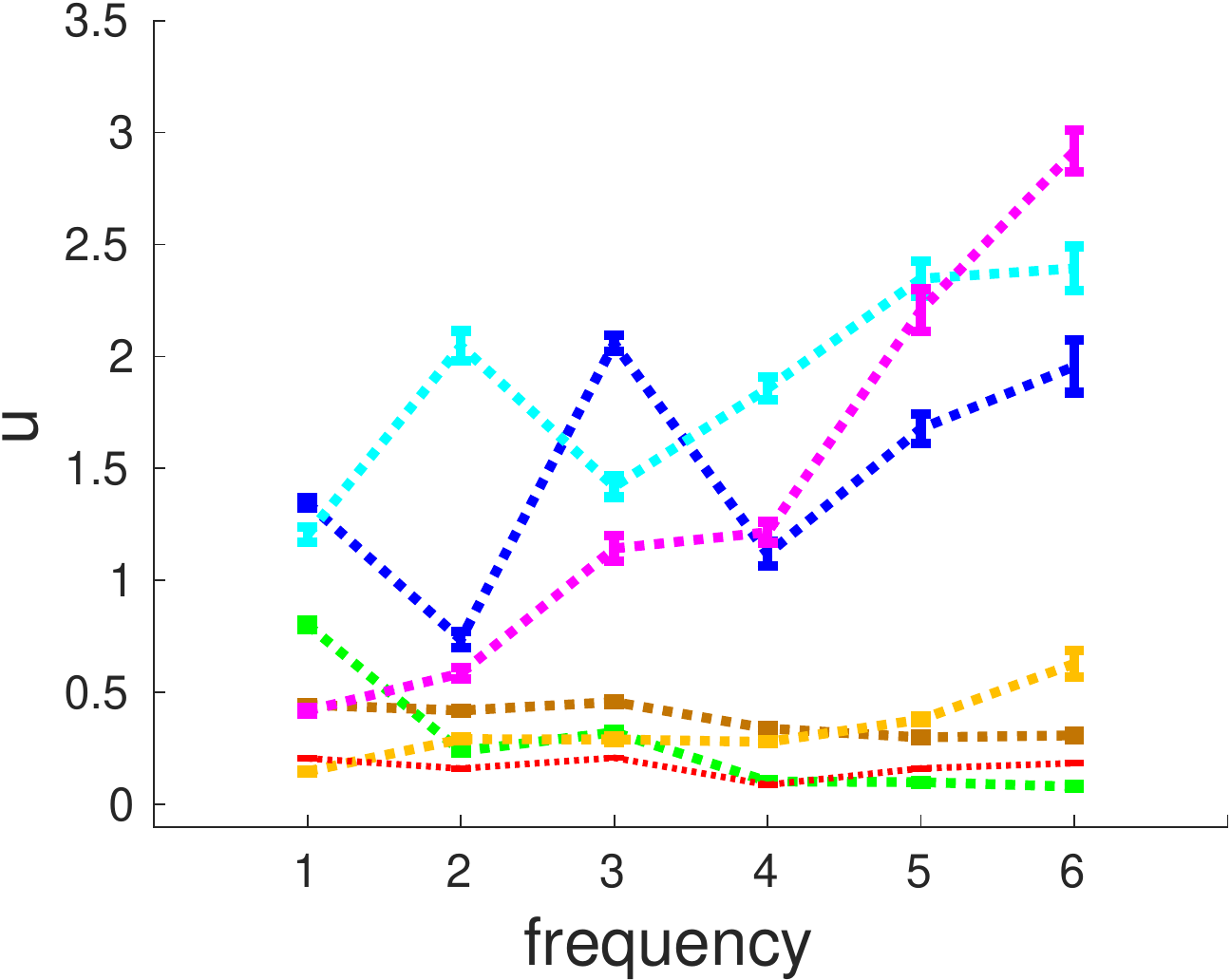}
         \caption{$x=-2.38$}
     \end{subfigure}
     \hfill
     \begin{subfigure}[b]{0.3\textwidth}
         \centering
         \includegraphics[width=\textwidth]{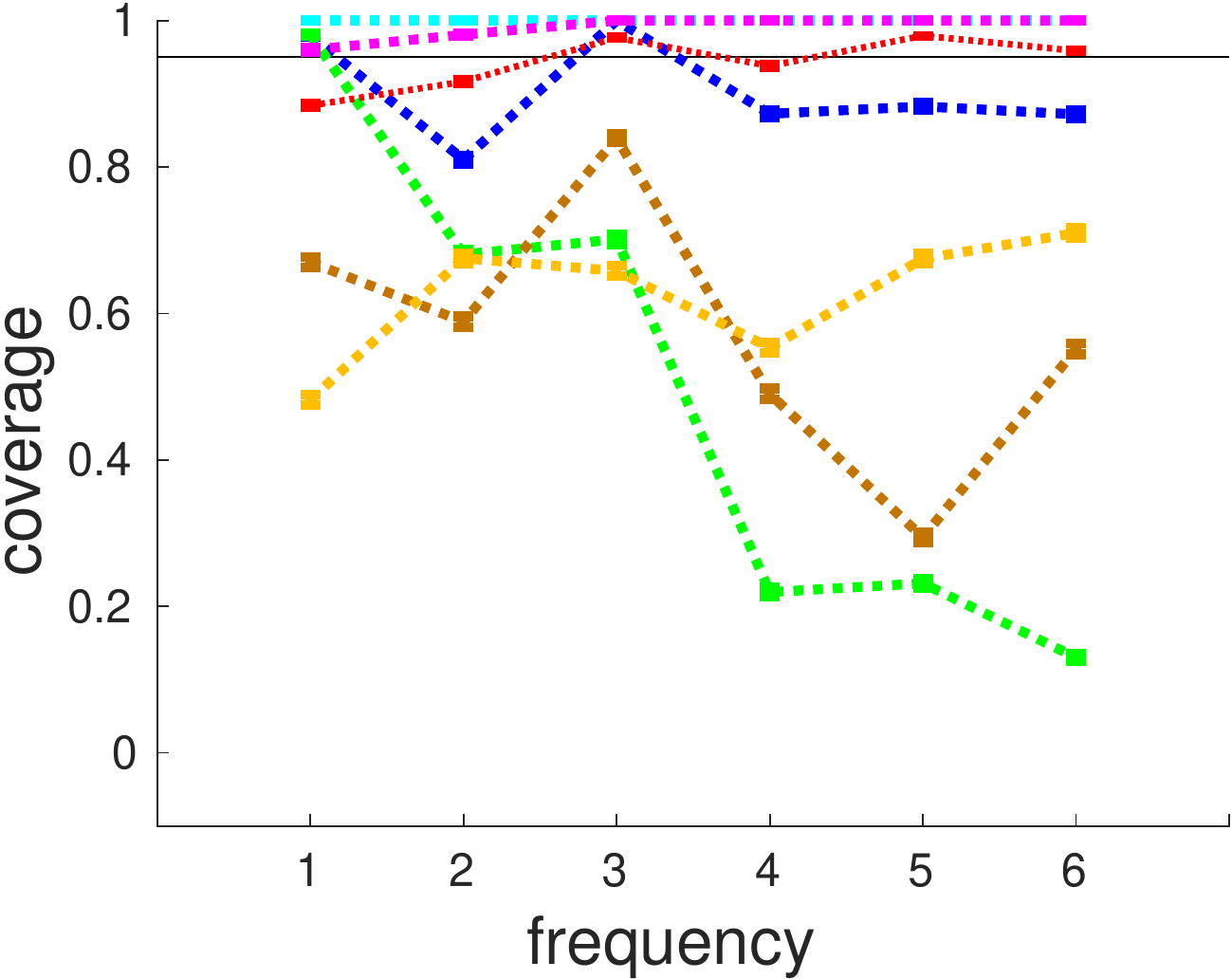}
         \caption{$x=-2.38$}
     \end{subfigure}\\
      \begin{subfigure}[b]{0.3\textwidth}
         \centering
         \includegraphics[width=\textwidth]{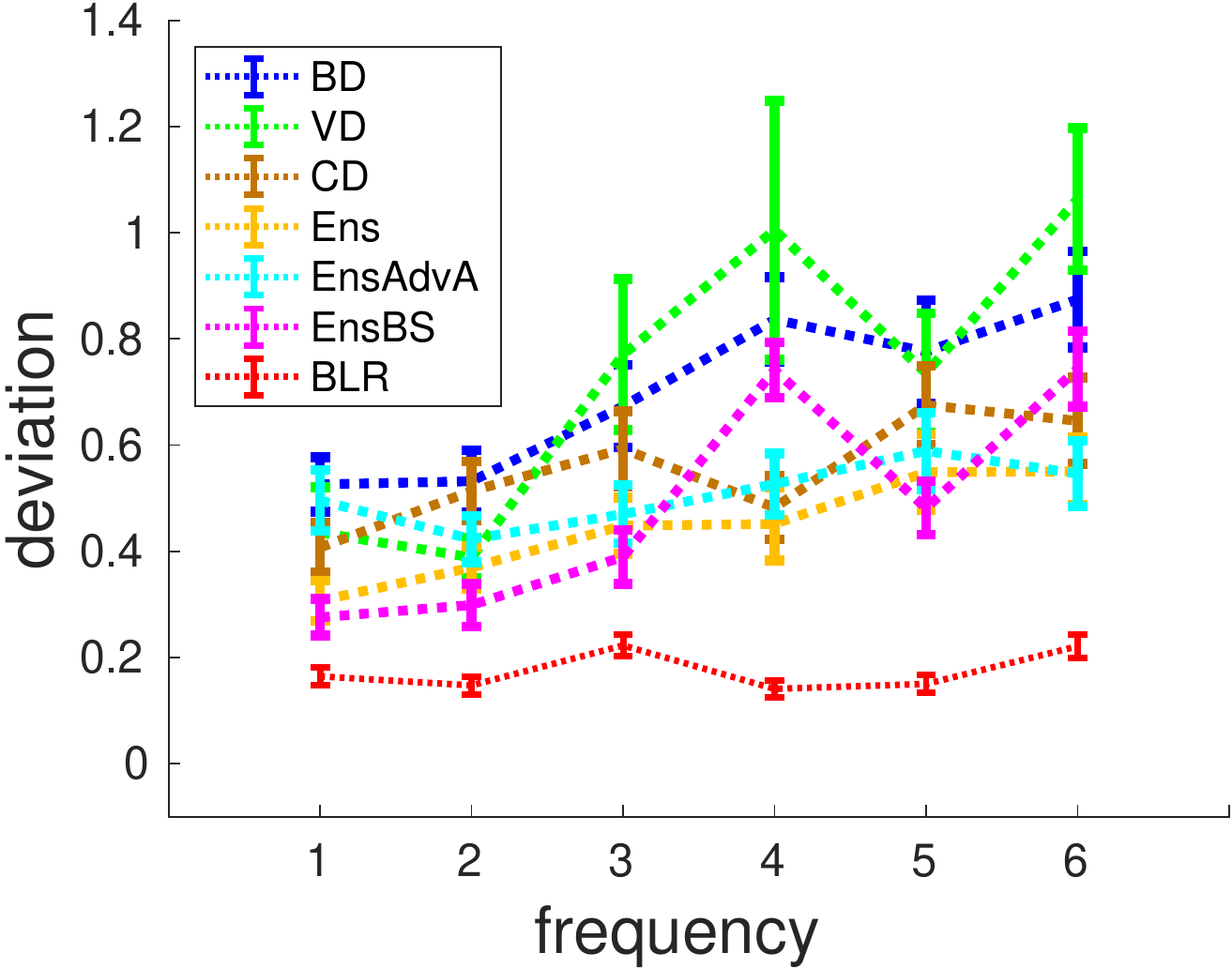}
         \caption{$x=1.2$}
     \end{subfigure}
     \hfill
     \begin{subfigure}[b]{0.3\textwidth}
         \centering
         \includegraphics[width=\textwidth]{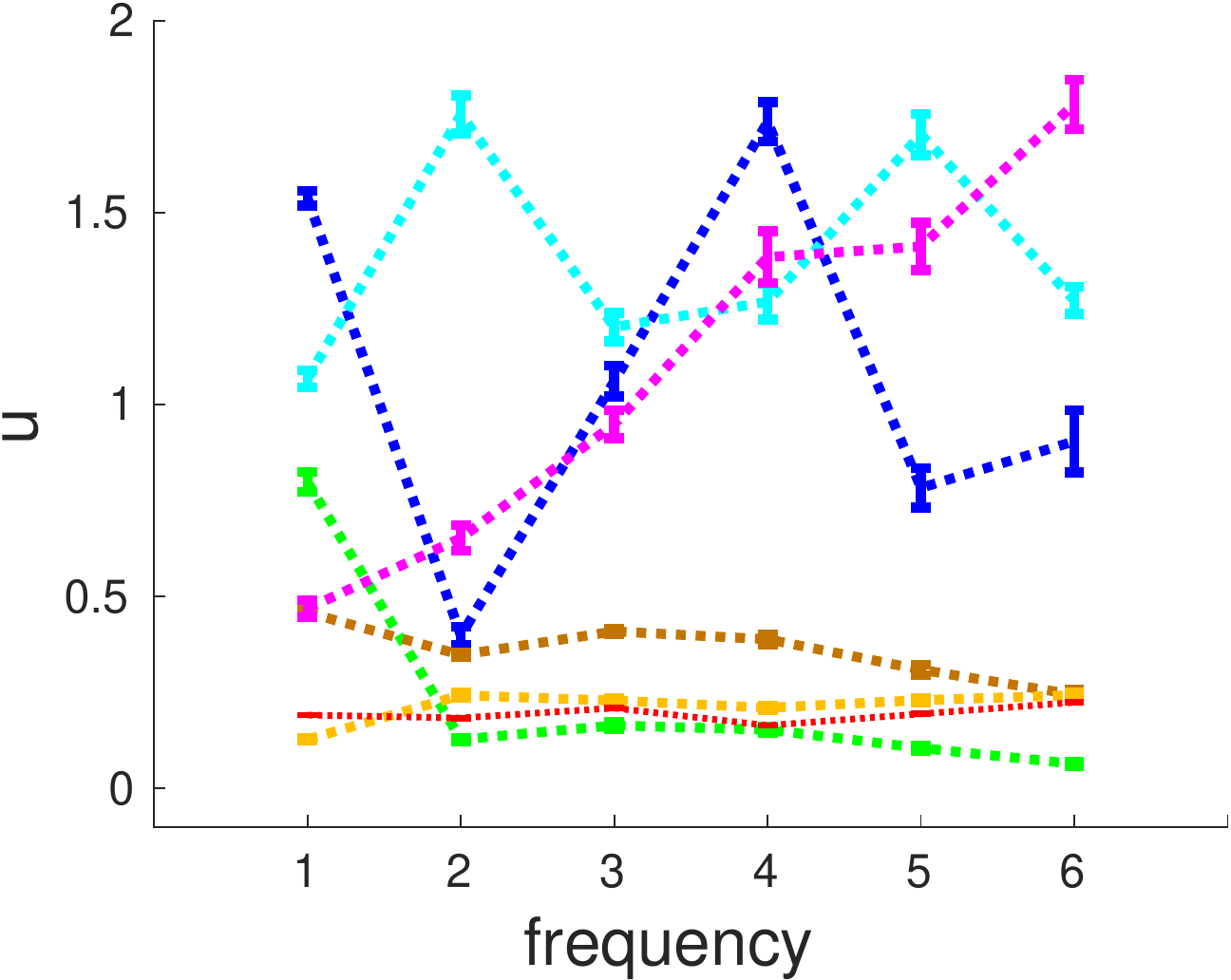}
         \caption{$x=1.2$}
     \end{subfigure}
     \hfill
     \begin{subfigure}[b]{0.3\textwidth}
         \centering
         \includegraphics[width=\textwidth]{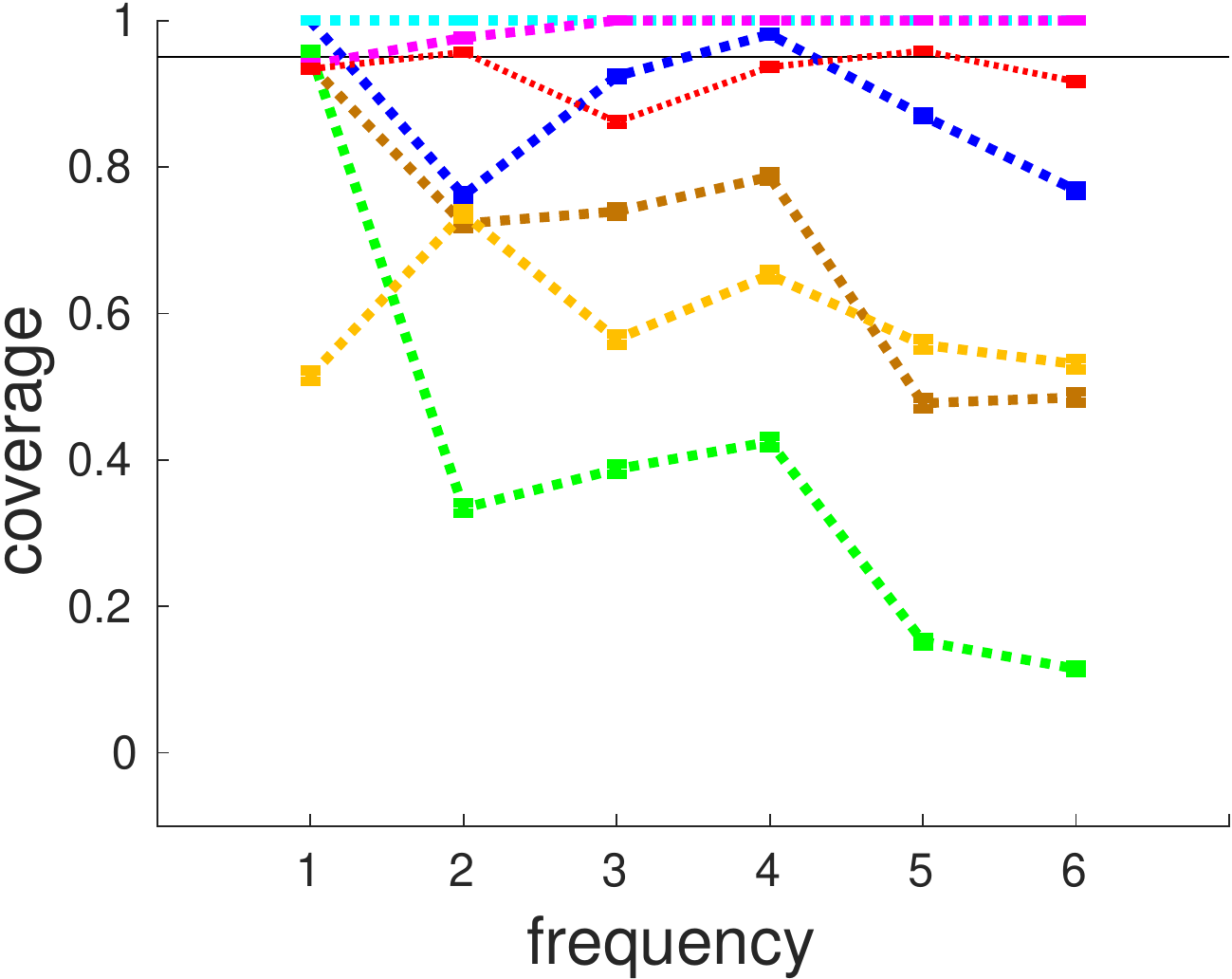}
         \caption{$x=1.2$}
     \end{subfigure}\\
      \begin{subfigure}[b]{0.3\textwidth}
         \centering
         \includegraphics[width=\textwidth]{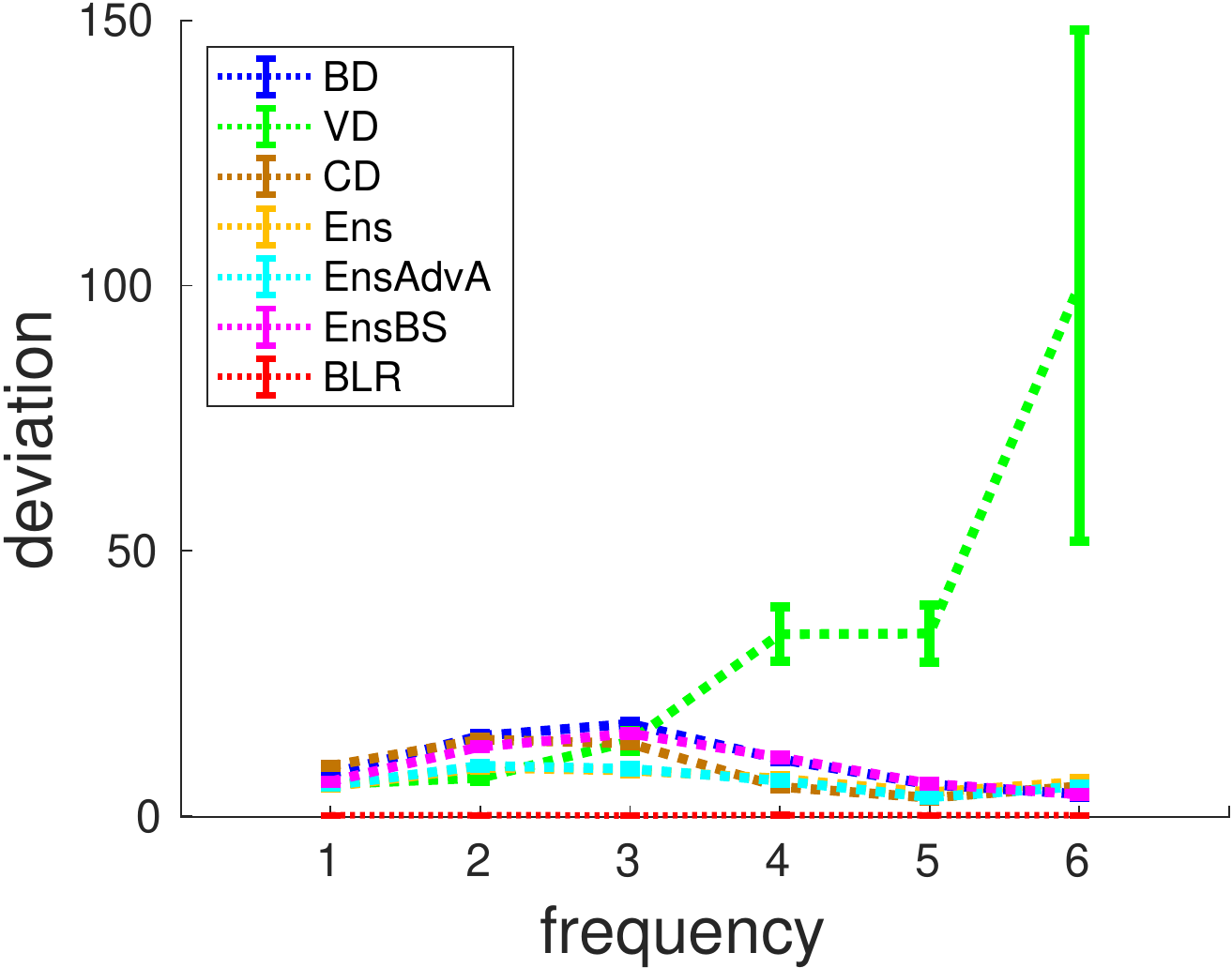}
         \caption{$x=-5.11$}
     \end{subfigure}
     \hfill
     \begin{subfigure}[b]{0.3\textwidth}
         \centering
         \includegraphics[width=\textwidth]{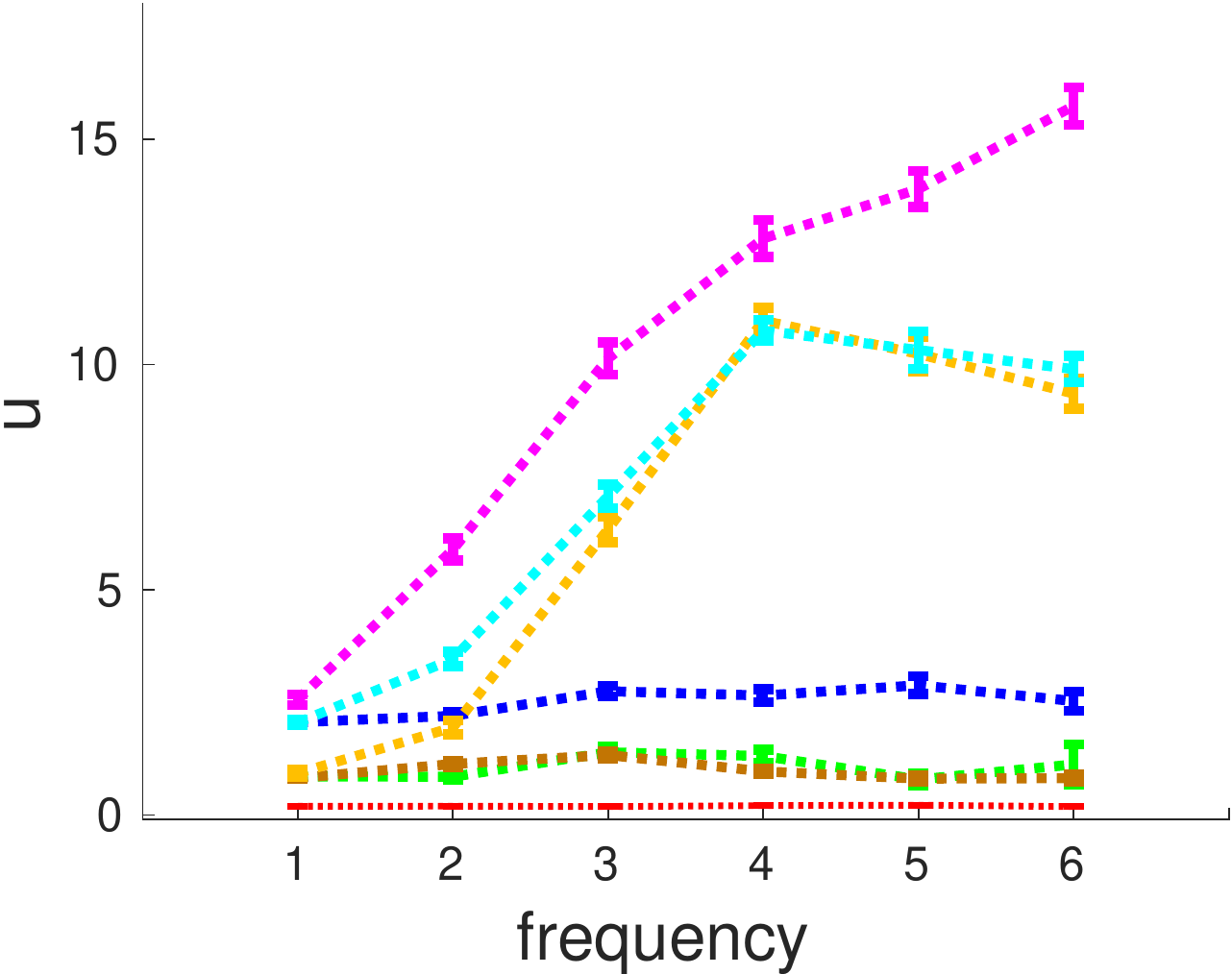}
         \caption{$x=-5.11$}
     \end{subfigure}
     \hfill
     \begin{subfigure}[b]{0.3\textwidth}
         \centering
         \includegraphics[width=\textwidth]{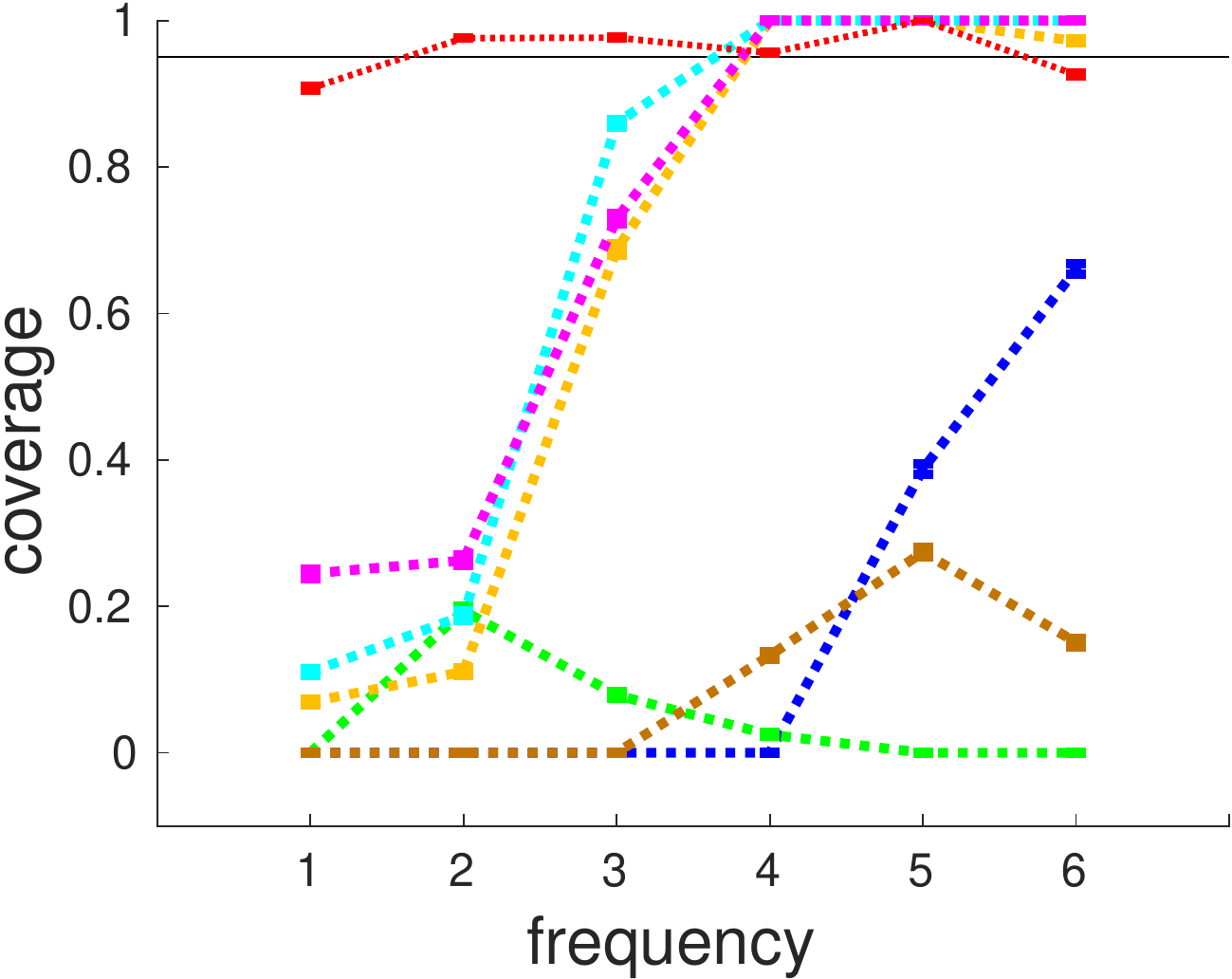}
         \caption{$x=-5.11$}
     \end{subfigure}
        \caption{Experiment \ref{E1} for various choices of the frequency $f_{\mathrm{main}}$.
        The three columns show for the anchor model and different methods from deep learning the deviation from the ground truth (left column), the uncertainty (middle column) and the coverage  of the ground truth by the uncertainty times 1.96 (right column) for three different inputs $x=-2.38$ (first row), $x=1.2$ (second row) and $x=-5.11$ (third row, out-of-distribution) 
        and different choices of the frequency $f_{\mathrm{main}}$ (on the abscissa). The plots were generated with 50 runs. For the deviation and uncertainty the mean together with its standard errors are plotted.} 
        \label{fig:_3planes_of_dice_iod_ood}
\end{figure}

The BLR-solution (red in fig. \ref{fig:_3planes_of_dice_iod_ood}) provides a complexity independent coverage probability of $0.95$ , i.e. $95\%$ of the ground truth is covered by the estimate $\pm 1.96$ times the standard deviation, as expected from the theory. At the same time, the deviation of the prediction of the BLR from the ground truth is quite low. The uncertainty is also comparably low (compare fig. \ref{fig:_3planes_of_dice_iod_ood}), but large enough to ensure $95\%$ coverage.

As motivated by lemma \ref{lem:optimality} we can, in a certain sense, consider a deep learning method as ``optimal'' if it reaches the BLR-solution in all three characteristics: deviation, uncertainty and coverage. Every method under test in the benchmarking framework should thus strive for the goal to reach the corner of a cube, defined by these three quantities, where the BLR-solution is located. Howevever, since the number of parameters in a deep regression drastically exceeds that of the BLR model, we cannot expect a deep regression approach to actually match the BLR-solution.

In fig. \ref{fig:_3planes_of_dice_iod_ood}, row 1 and 2, we observe an increasing deviation by increasing complexity for all methods, except for BLR, for inputs of the in-distribution (row 1 $\rightarrow$ $x=-2.38$, row 2 $\rightarrow$ $x=1.20$).  At the same time we can observe, in the subplots of the second column, that the size of the associated uncertainty over increasing frequency is only slightly increasing or constant for the approaches based on dropout (BD, VD and CD). As predictable from subplot (a) and (b) the coverage, depicted in (c), is poor for the Dropout-based methods and substantially lower as the $0.95$ provided by the anchor model. The non-standard ensemble methods (EnsAdvA and EnsBS), on the other hand, deliver a coverage near $0.95$ but overestimate for larger complexities their uncertainty, which be can be seen by the coverage of almost 1.0. Note, that for the considered scenario the uncertainty of the anchor model is below the uncertainties of all other models.

Row 3 of fig. \ref{fig:_3planes_of_dice_iod_ood} shows the results for the case of one out-of-distribution input test data ($x=-5.11$) in the evaluation of the method. We encounter a slightly different behavior. The deviation from the ground truth is not increasing with increasing complexity for all methods, except for VD. The size of the associated uncertainty is not increasing and too low to deliver a sufficient coverage, except for the ensemble methods. The reason for this behavior seems to lie in the poor or not sufficient quality of the prediction in the out-of-distribution range together with an underestimated uncertainty.\\

In regard to the figs. above, all ensemble methods absorb the increasing deviation over increasing complexity with an increasing uncertainty, especially  EnsAdvA and EnsBS. This results in a good, but partially overestimated, coverage near the value provided by the anchor method (BLR). All methods based on the dropout technique (BD, CD and VD) underestimate the uncertainty with simultaneously increasing deviations over increasing complexity, whereby a good coverage (ideal $0.95$) is not possible.\\ 

\begin{figure}[t!]
     \centering
     \begin{subfigure}[b]{0.24\textwidth}
         \centering
         \includegraphics[width=\textwidth]{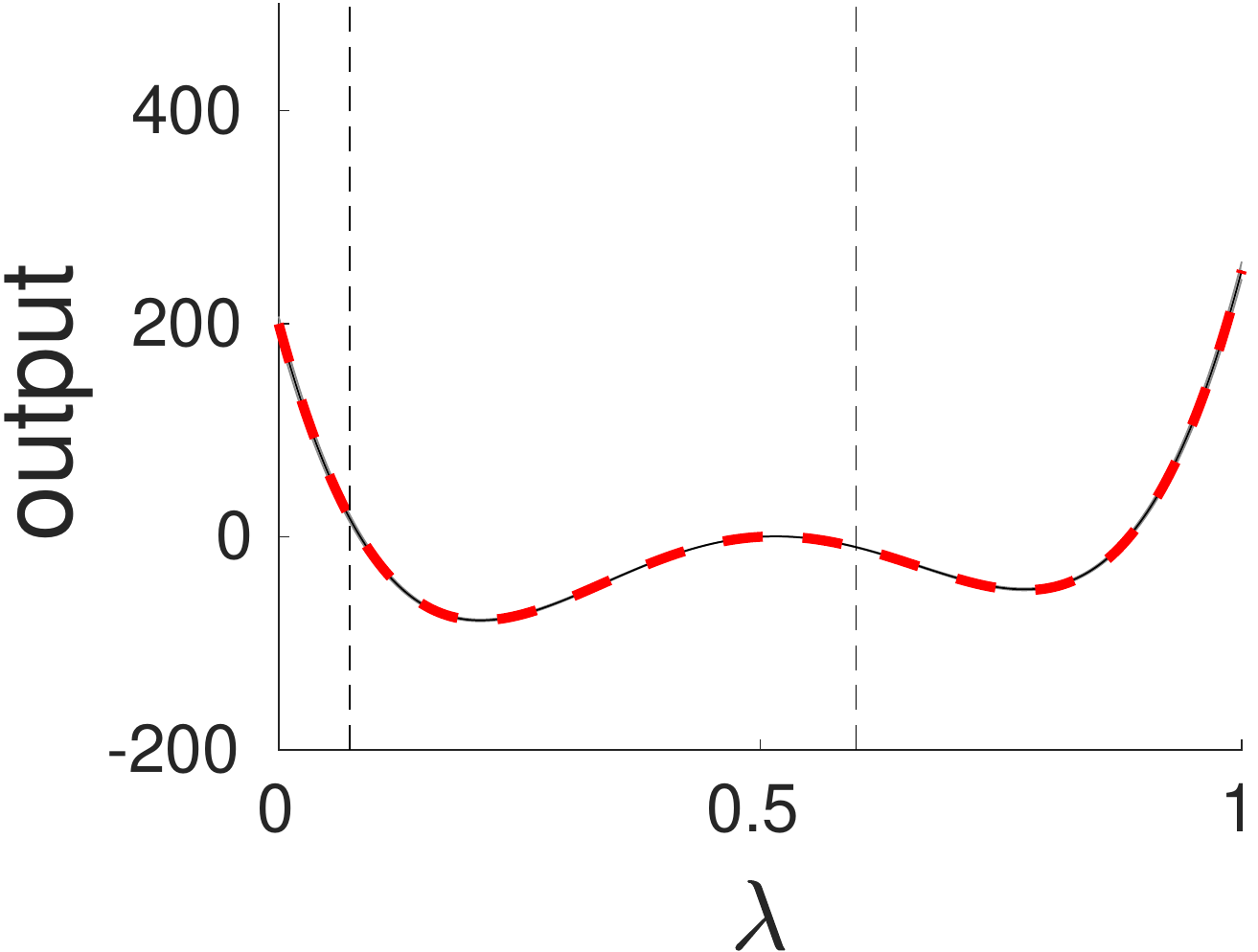}
         \caption{BLR}
     \end{subfigure}
     \hfill
     \begin{subfigure}[b]{0.24\textwidth}
         \centering
         \includegraphics[width=\textwidth]{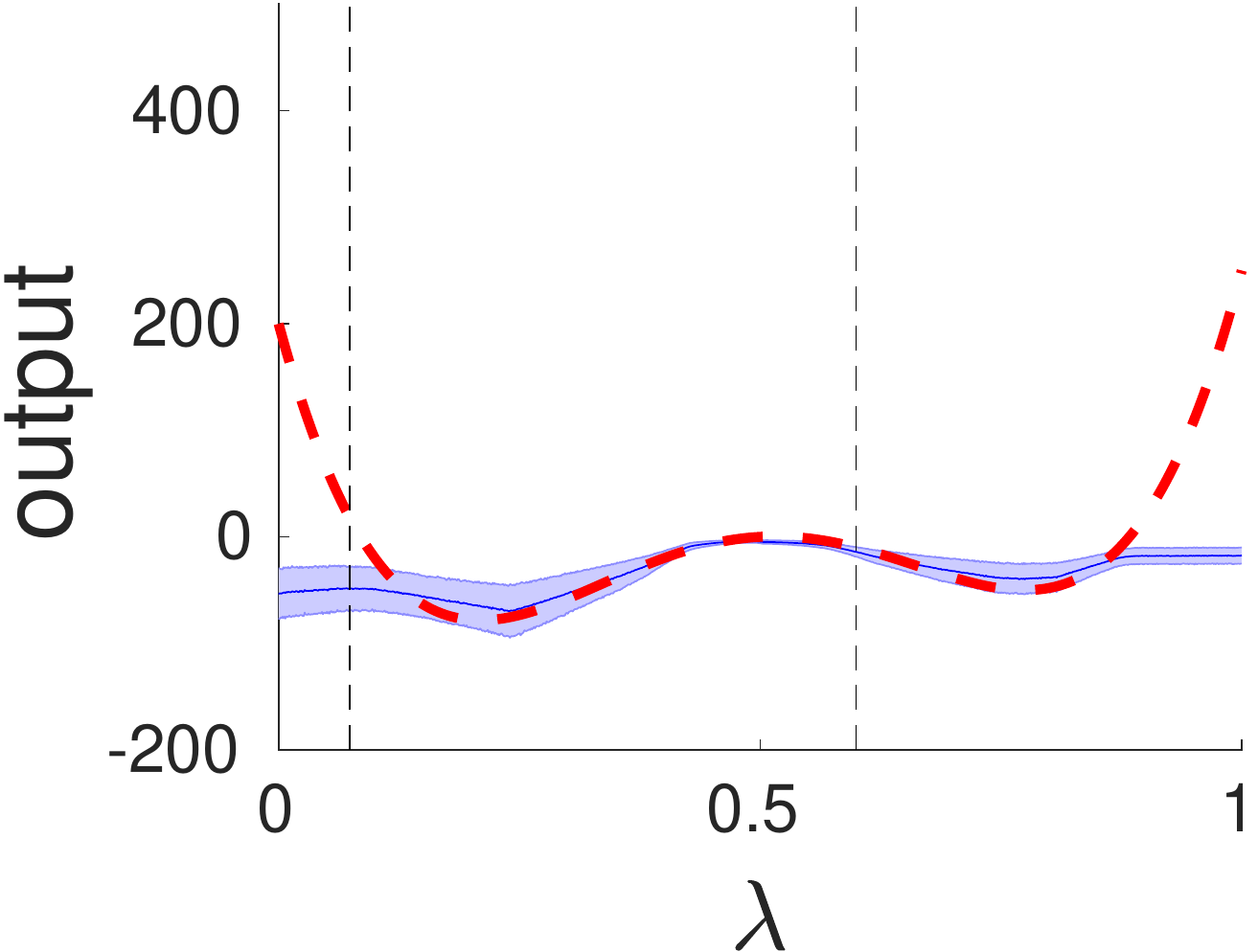}
         \caption{BD}
     \end{subfigure}
     \hfill
     \begin{subfigure}[b]{0.24\textwidth}
         \centering
         \includegraphics[width=\textwidth]{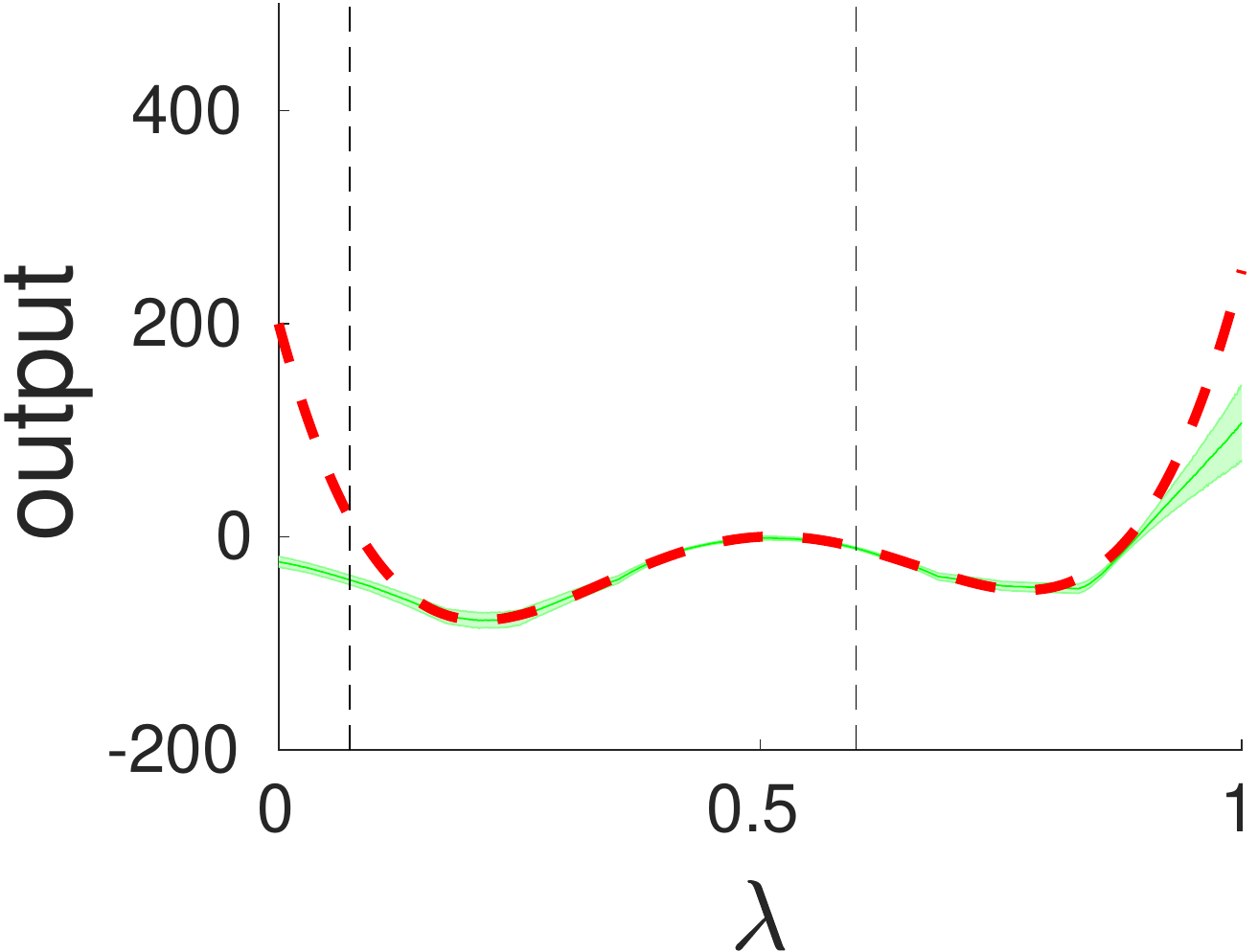}
         \caption{VD}
     \end{subfigure}
     \hfill
     \begin{subfigure}[b]{0.24\textwidth}
         \centering
         \includegraphics[width=\textwidth]{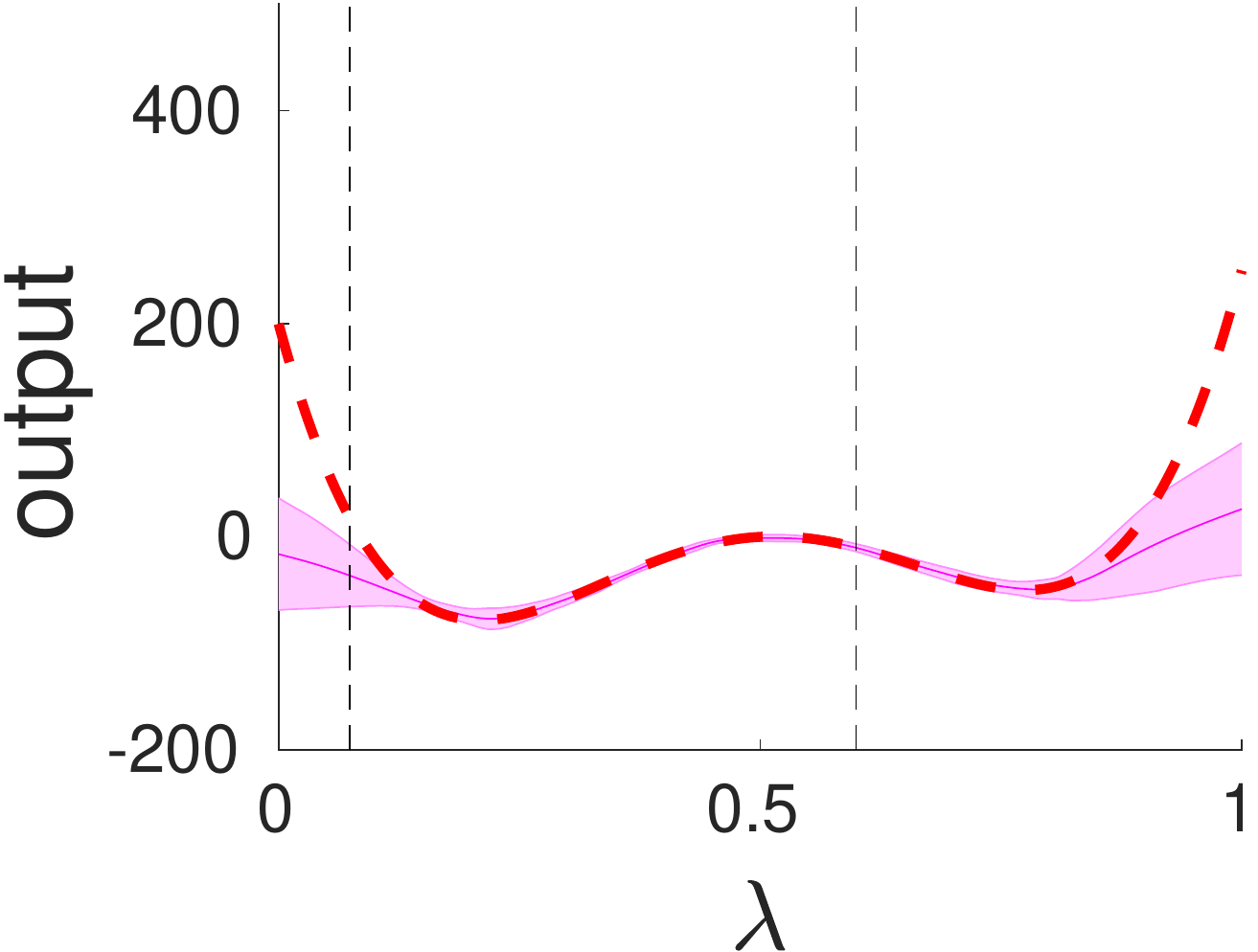}
         \caption{EnsBS}
     \end{subfigure}
        \caption{
        Experiment \ref{E2} in dimension $d=2$. Predictions (solid lines) and according uncertainties times 1.96 (shaded areas) of the anchor model BLR (subplot (a) - uncertainties are very small, barely visible) and different methods in deep learning (subplots (b)-(d)) along a diagonal of the $d$-dimensional hypercube (cf. section \ref{sec:experiments}), together with the ground truth (red dashed line). The abbreviations for the methods in the beginning of section \ref{sec:results}. 
        }
        \label{fig:exampleE2_dimlow}
\end{figure}

\begin{figure}[]
     \centering
     \begin{subfigure}[b]{0.24\textwidth}
         \centering
         \includegraphics[width=\textwidth]{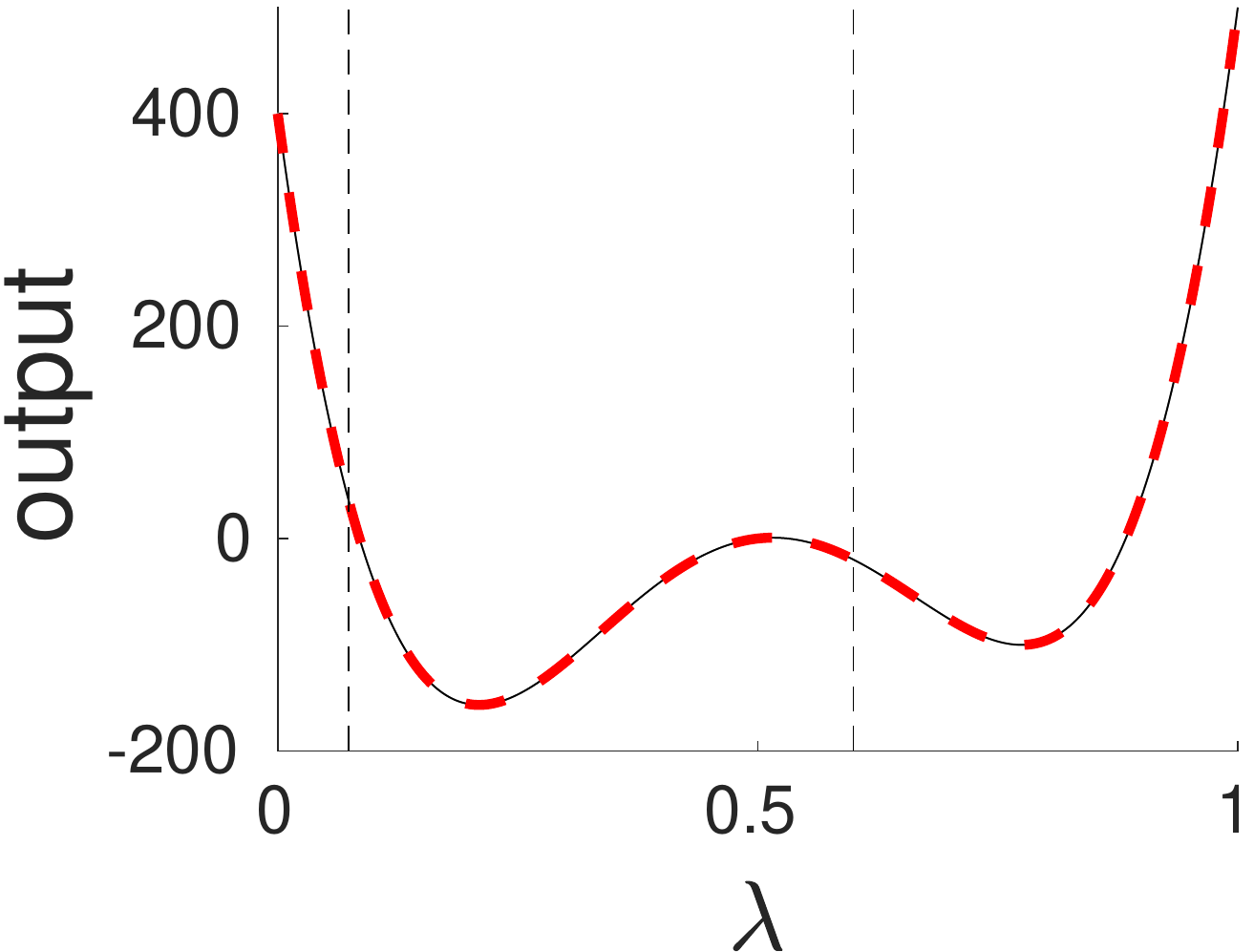}
         \caption{BLR}
     \end{subfigure}
     \hfill
     \begin{subfigure}[b]{0.24\textwidth}
         \centering
         \includegraphics[width=\textwidth]{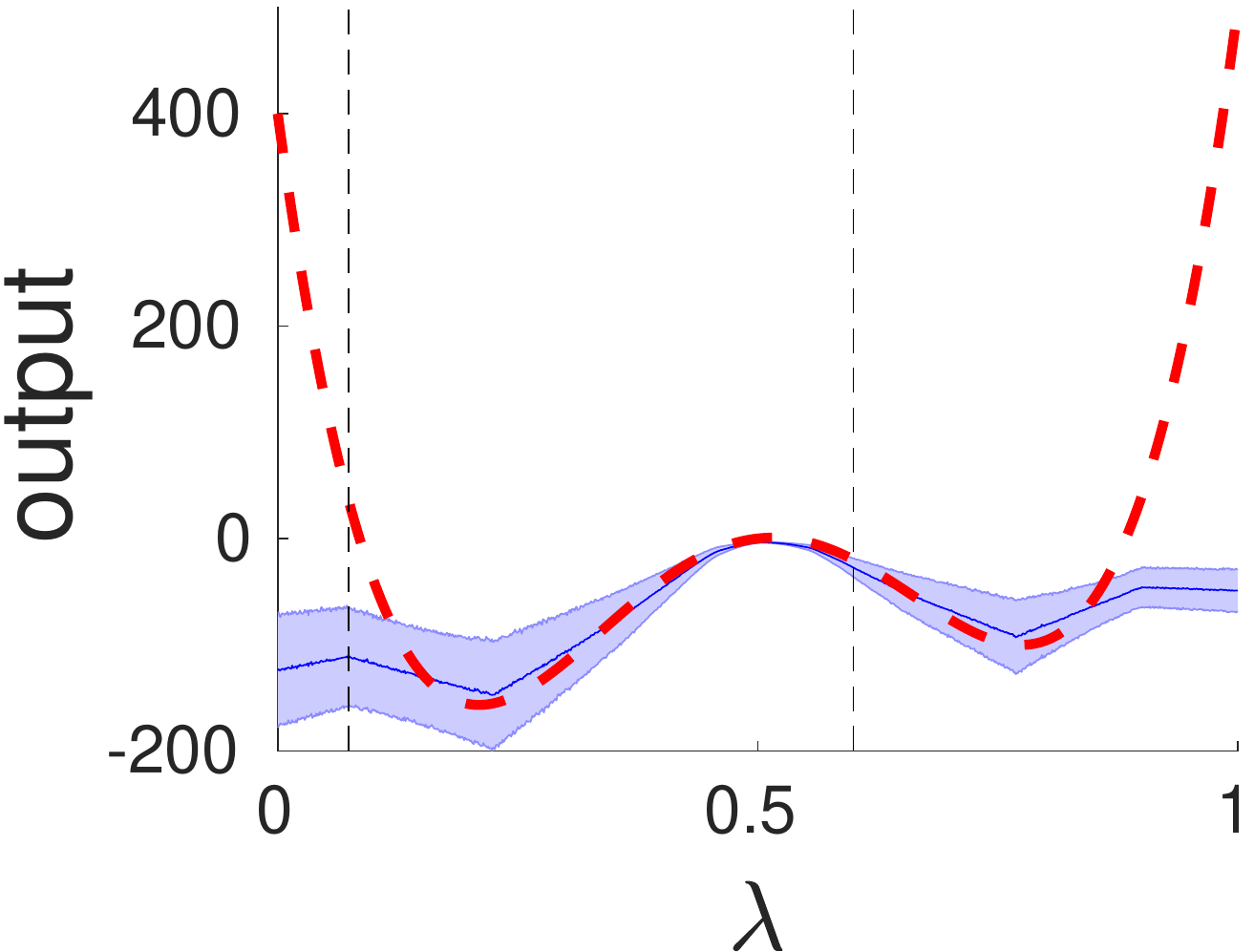}
         \caption{BD}
     \end{subfigure}
     \hfill
     \begin{subfigure}[b]{0.24\textwidth}
         \centering
         \includegraphics[width=\textwidth]{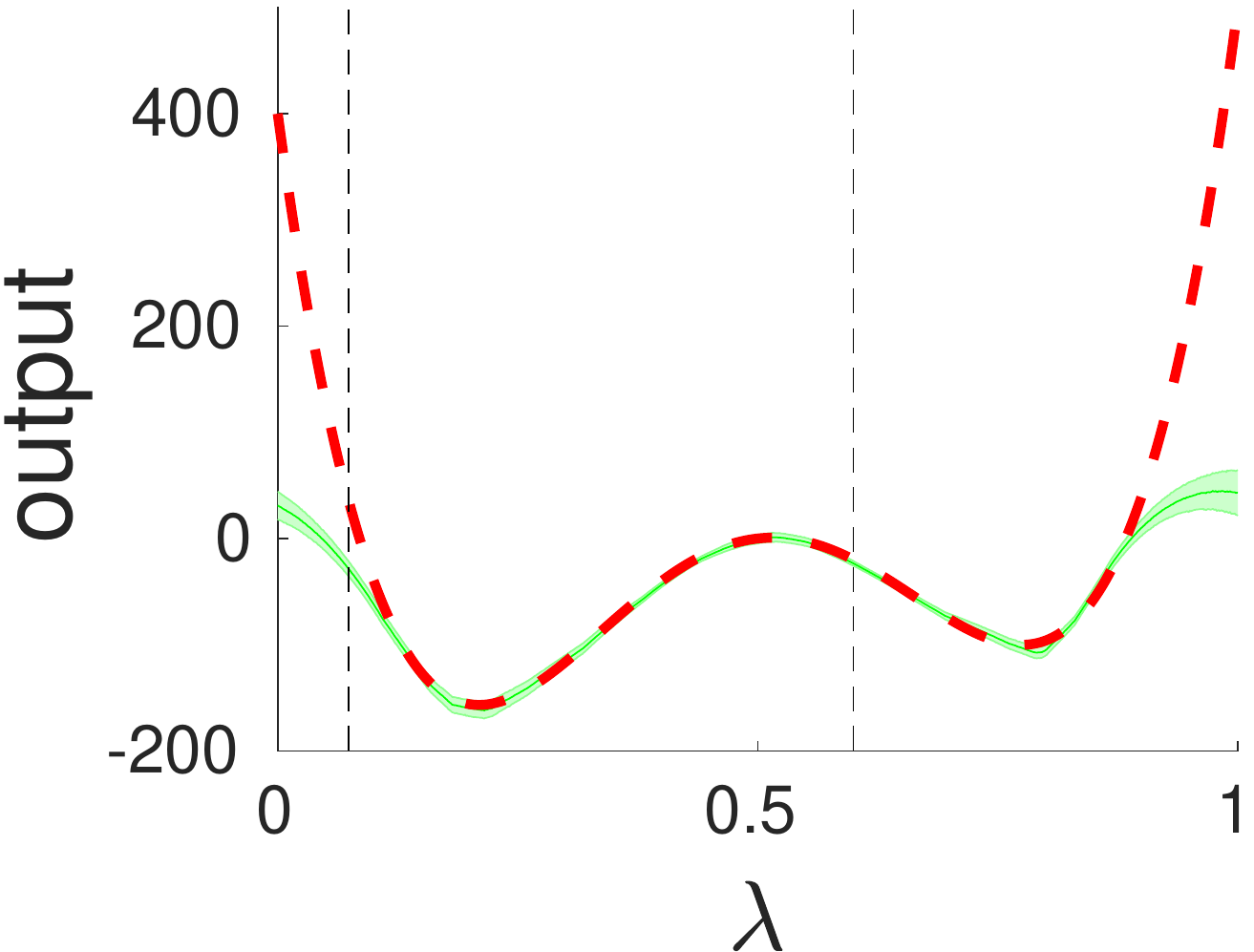}
         \caption{VD}
     \end{subfigure}
     \hfill
     \begin{subfigure}[b]{0.24\textwidth}
         \centering
         \includegraphics[width=\textwidth]{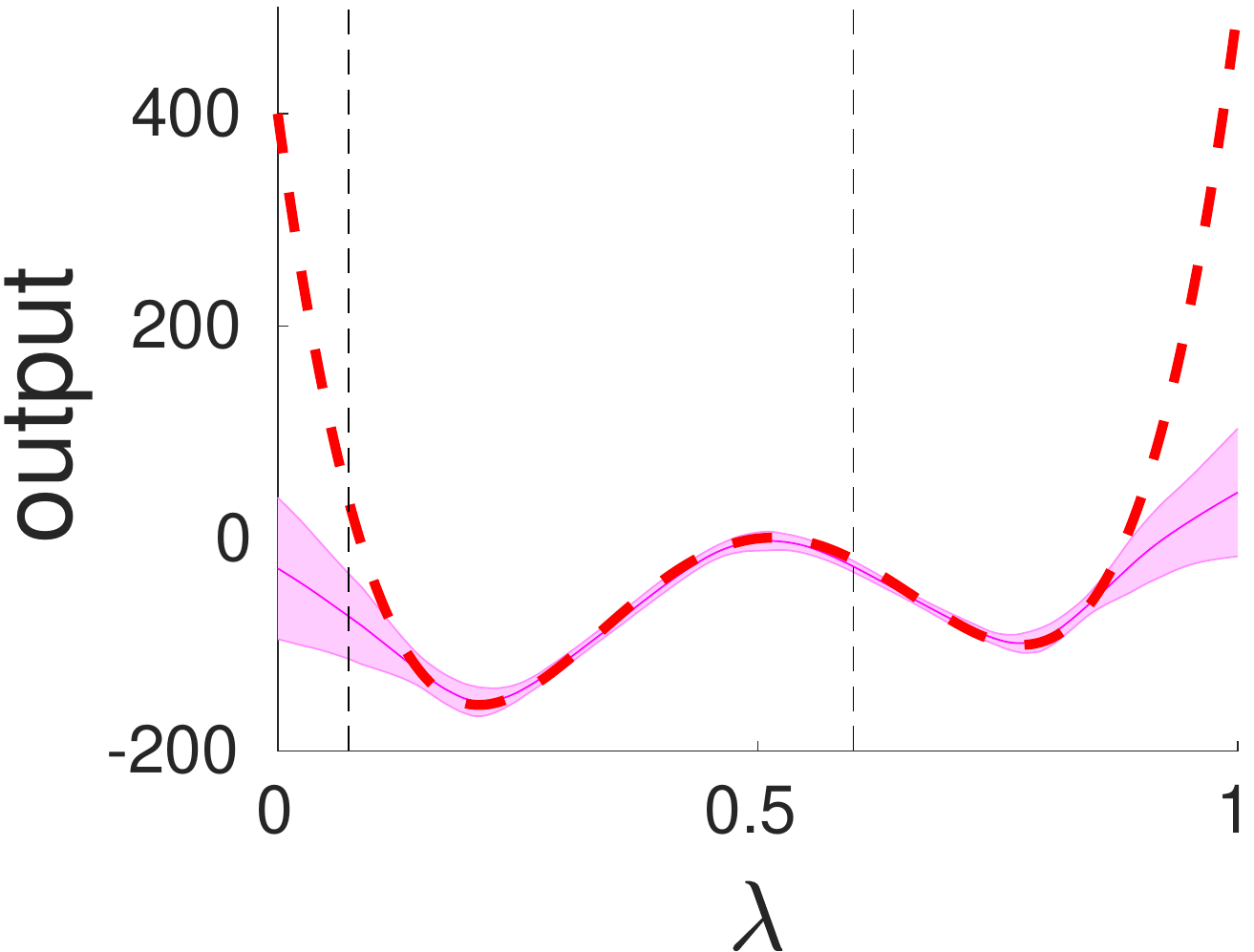}
         \caption{EnsBS}
     \end{subfigure}
        \caption{
        Experiment \ref{E2} in dimension $d=4$. Predictions (solid lines) and according expanded uncertainties (shaded areas) of the anchor model BLR (subplot (a) - uncertainties are very small, barely visible) and different methods in deep learning (subplots (b)-(d)) along a diagonal of the $d$-dimensional hypercube (cf. section \ref{sec:experiments}), together with the ground truth (red dashed line). The abbreviations for the methods are as in the beginning of section \ref{sec:results}.
        }
        \label{fig:exampleE2_dimhigh}
\end{figure}

\begin{figure}[t!]
     \centering
     \begin{subfigure}[b]{0.3\textwidth}
         \centering
         \includegraphics[width=\textwidth]{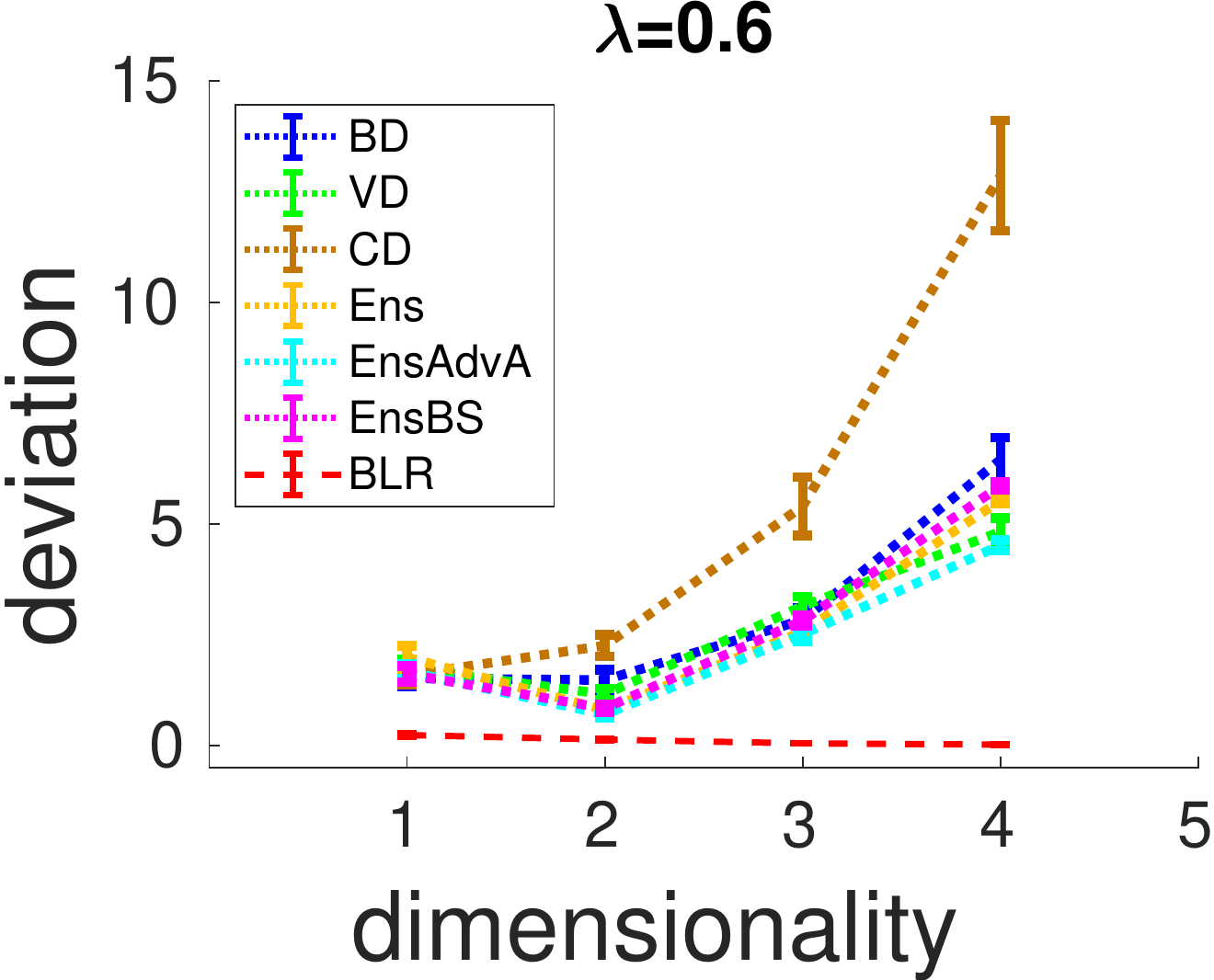}
         \caption{}
     \end{subfigure}
     \hfill
     \begin{subfigure}[b]{0.3\textwidth}
         \centering
         \includegraphics[width=\textwidth]{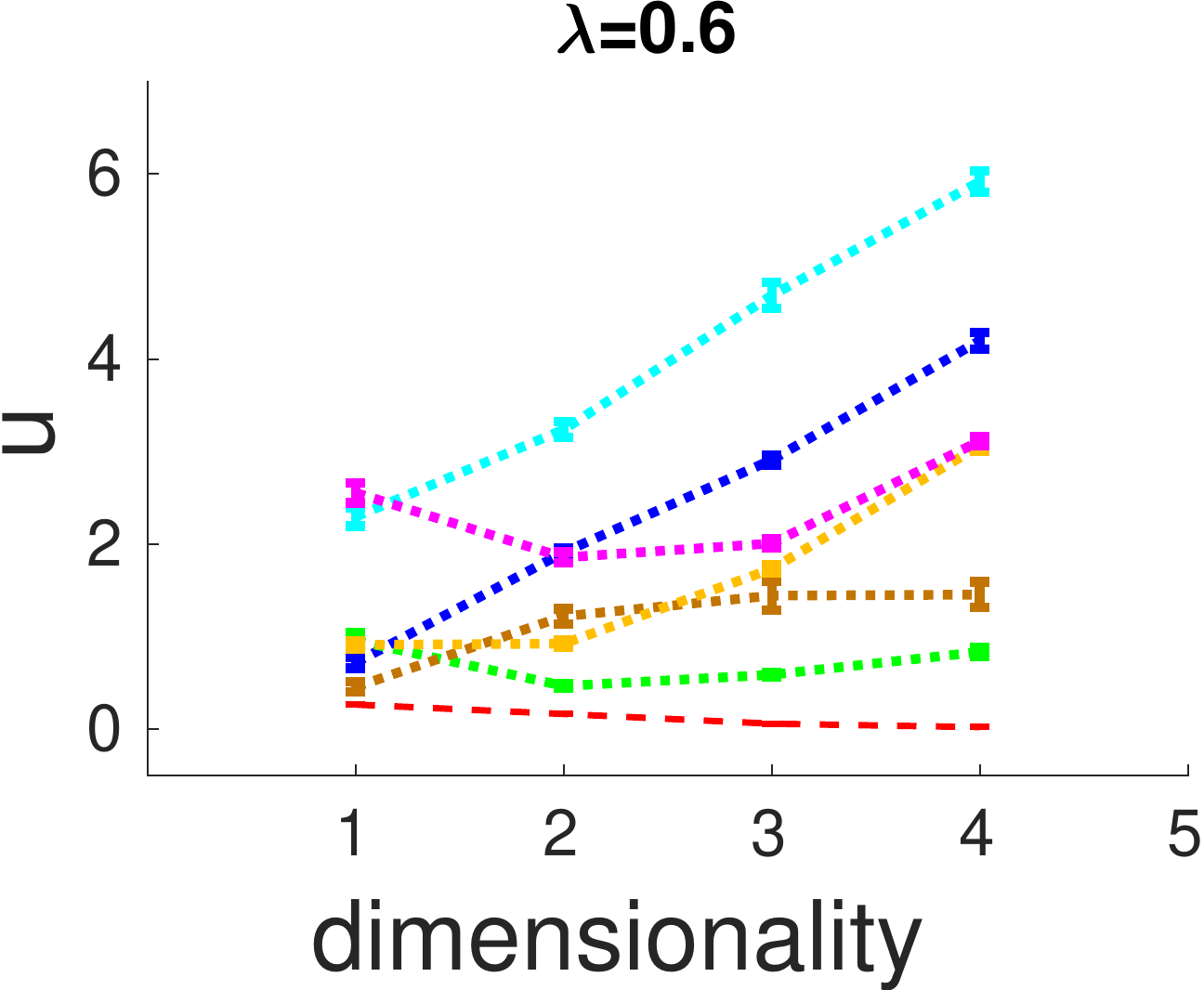}
         \caption{}
     \end{subfigure}
     \hfill
     \begin{subfigure}[b]{0.3\textwidth}
         \centering
         \includegraphics[width=\textwidth]{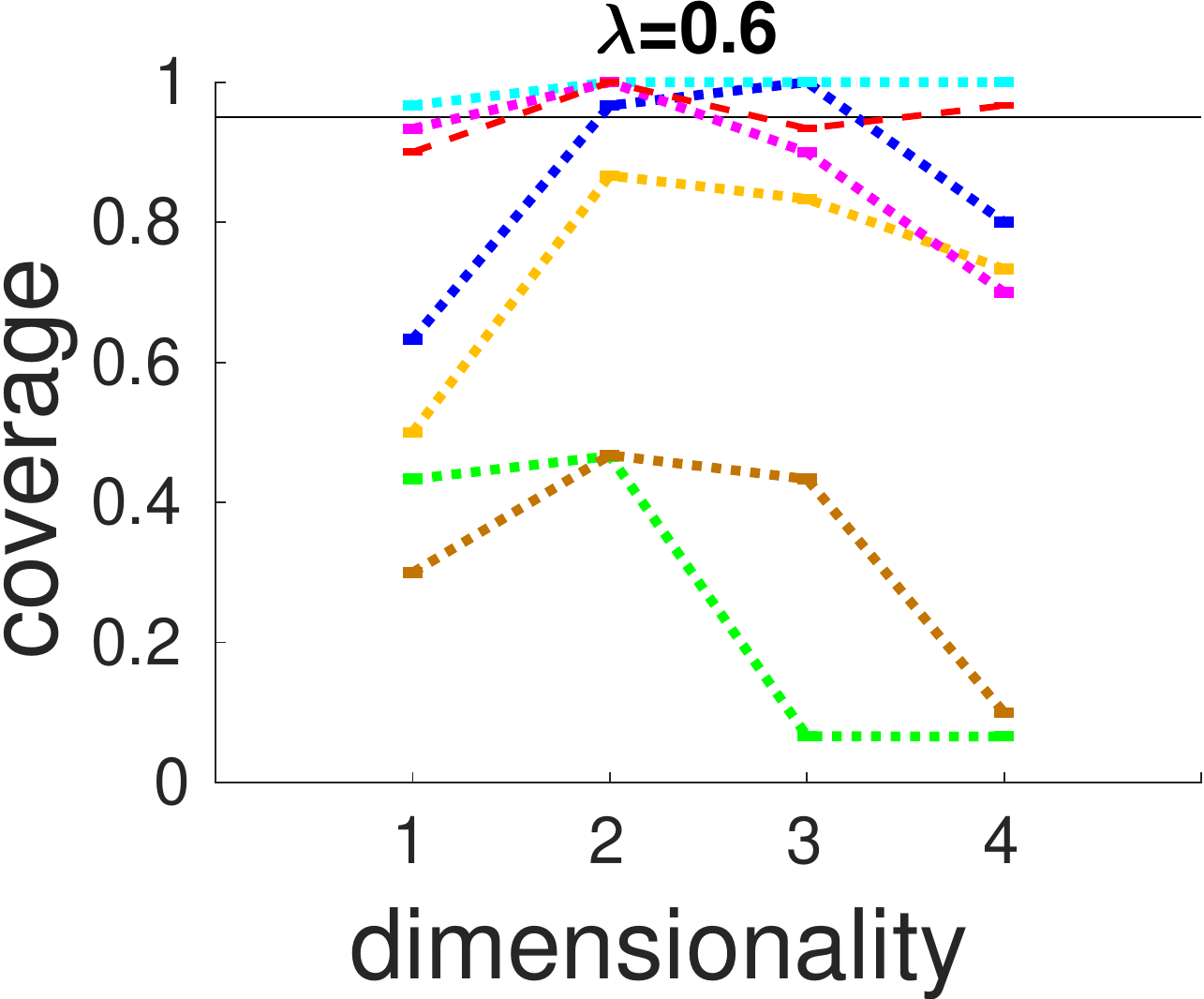}
         \caption{}
     \end{subfigure}\\
      \begin{subfigure}[b]{0.3\textwidth}
         \centering
         \includegraphics[width=\textwidth]{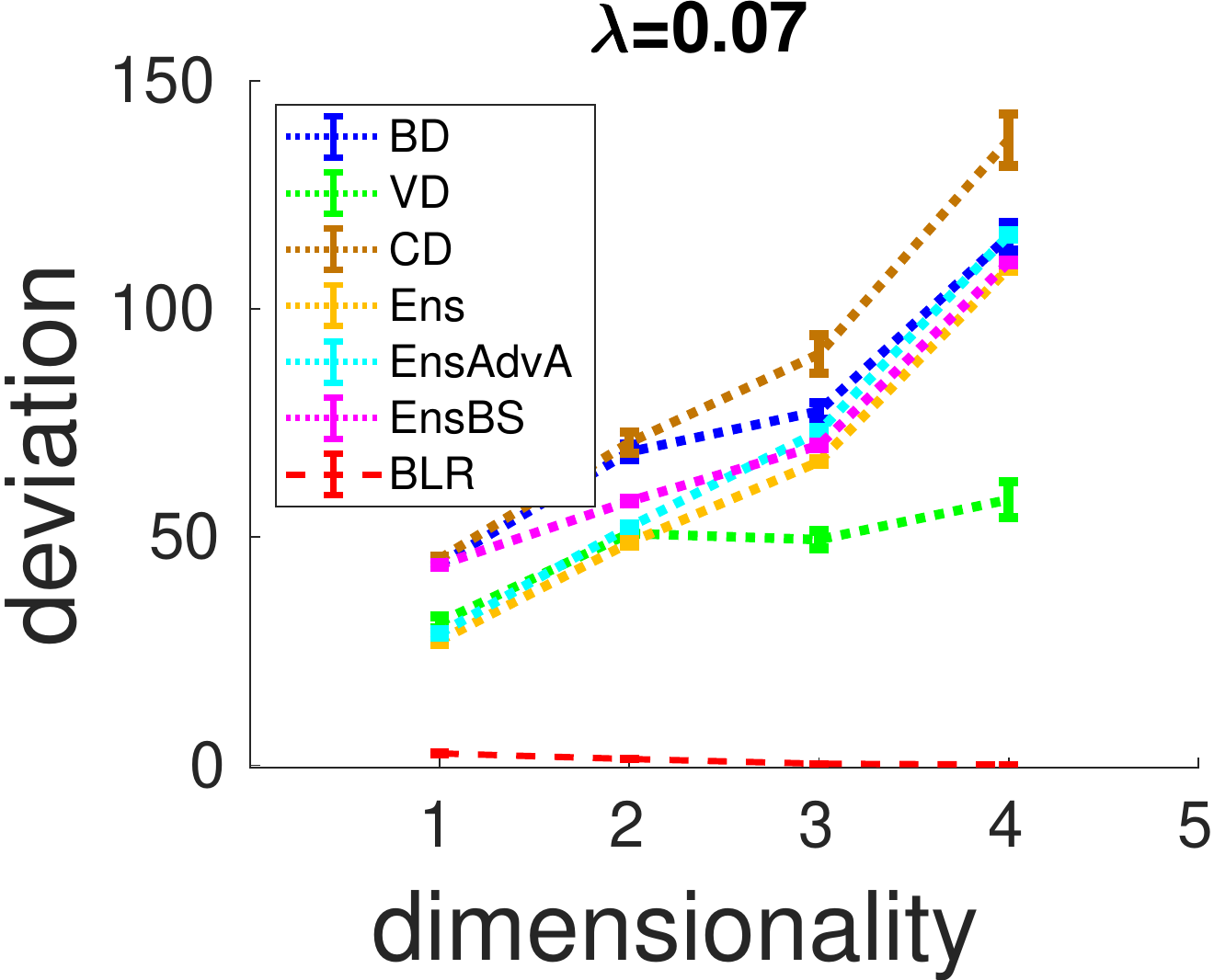}
         \caption{}
     \end{subfigure}
     \hfill
     \begin{subfigure}[b]{0.3\textwidth}
         \centering
         \includegraphics[width=\textwidth]{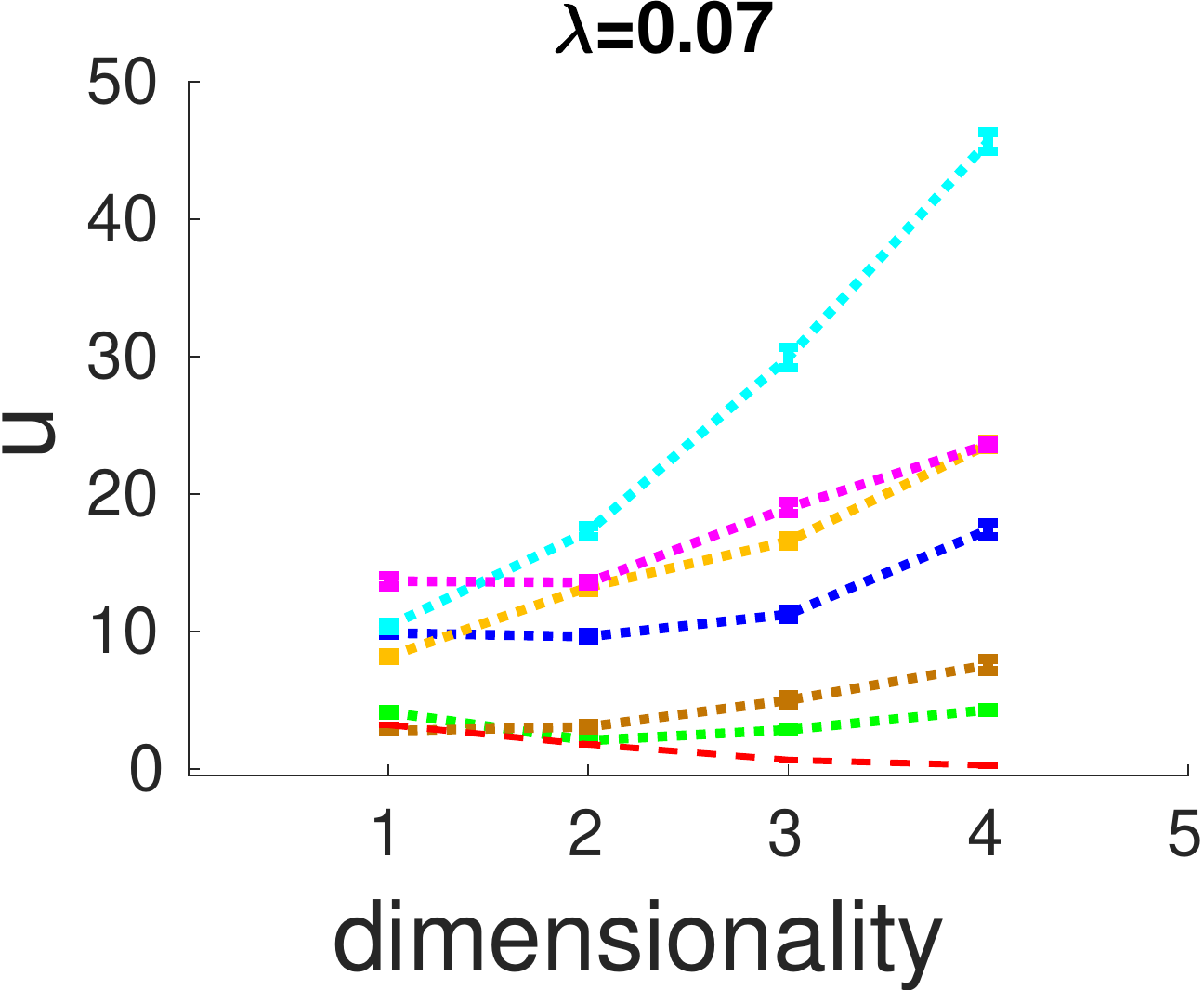}
         \caption{}
     \end{subfigure}
     \hfill
     \begin{subfigure}[b]{0.3\textwidth}
         \centering
         \includegraphics[width=\textwidth]{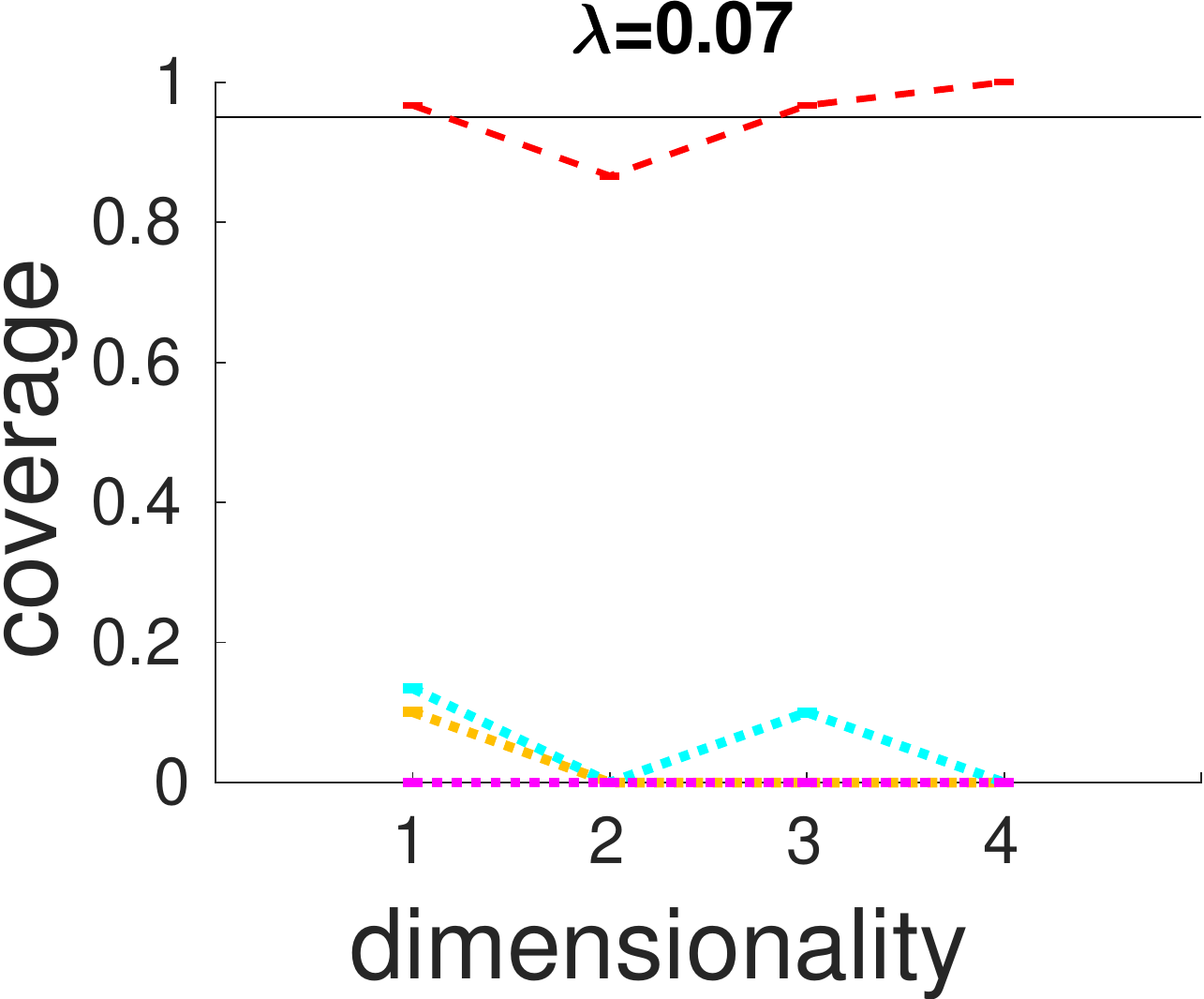}
         \caption{}
     \end{subfigure}
     \caption{Experiment \ref{E2} for various choices of the dimensionality $d$. The three columns show for the anchor model and different methods from deep learning the deviation from the ground truth (left column), the uncertainty (middle column) and the coverage of the ground truth by the uncertainty times 1.96 (right column) for two different inputs $\lambda=0.5996$ (first row, ind-distribution, cf. fig. \ref{fig:exampleE2_dimlow}) and $\lambda=0.074$ (second row, out-of-distribution, cf. fig. \ref{fig:exampleE2_dimhigh})
     }
        \label{fig:_3planes_of_dice_dim_iod_ood}
\end{figure}

Fig. \ref{fig:exampleE2_dimlow} and \ref{fig:exampleE2_dimhigh} show the according results for the experiment \ref{E2}. The elements of the test set were chosen to contain only inputs that are located at the diagonal of the $d$-dimensional hypercube from $(-5,\ldots,-5)^T\in\mathbb{R}^{d}$ to $(5,\ldots,5)^T\in\mathbb{R}^{d}$, i.e. $x_{\lambda}= (1-\lambda)(-5,-5,\ldots,-5)^T+\lambda(5,5,\ldots,5)^T$ with $0\leq \lambda \leq 1$. In that way, the diagonal contains both, out-of-distribution and in-distribution inputs.
Fig. \ref{fig:_3planes_of_dice_dim_iod_ood} depicts the deviation from the ground truth, uncertainty and coverage for two different choices of $\lambda$, namely $\lambda=0.6$ for in-distribution and $\lambda=0.07$ for out-of-distribution. For both inputs $x_\lambda $ we observe an overestimation of the uncertainty, compared with the anchor model, for the ensemble methods with adversarial attacks (EnsAdvA) and bootstrapping (EnsBS).
Both ensemble methods cover for larger dimensions the increasing deviation from the ground truth with an increasing uncertainty, but the increase of uncertainty in dependence in the increasing dimensionality seems a bit too large, especially for EnsAdvA. For CD and VD we observe a decreasing coverage over increasing dimensionality. At the same time the deviation from the ground truth is strongly increasing for CD. The increasing uncertainty over increasing dimensionality is too low, so that the uncertainty does not compensate for the increasing error for this method. For BD we observe an increasing uncertainty with increasing dimensionality and a moderately increasing deviation, which leads to a coverage of round about $0.8$ over increasing dimensionality. For one representative input of the out-of-distribution range  (see subplots fig. \ref{fig:exampleE2_dimlow} and \ref{fig:exampleE2_dimhigh}) we observe for all DL-methods a poor coverage, i.e. we observe a poor behavior in the out-of-distribution range and at the same time an underestimated uncertainty.\\

For experiment \ref{E3}, fig. \ref{fig:E3_heatmaps} shows the deviation from the ground truth, the uncertainty and the coverage as a heat map for every input of the test set (regular grid in $x_{1}$ and $x_{2}$). The in-distribution region is located at $1 \leq x_{1} \leq 4$ and $1 \leq x_{2} \leq 4$. Every row represents a method, the first row contains results provided by BLR, row 2 and 3 contain results from BD and EnsBS. The deviation and uncertainty is shown as a decadic logarithm. The BLR method shows on every position of the input grid a coverage around $0.95$. BD shows in some regions of the input space a poor coverage, i.e. the associated uncertainty of the BD does not cover the deviation in this regions. The ensemble method using bootstrapping (EnsBS) only exhibits in the out-of-distribution range a coverage well below 0.95.

\begin{figure}[h!]
\centering%
\settoheight{\tempdima}{\includegraphics[width=.32\linewidth]{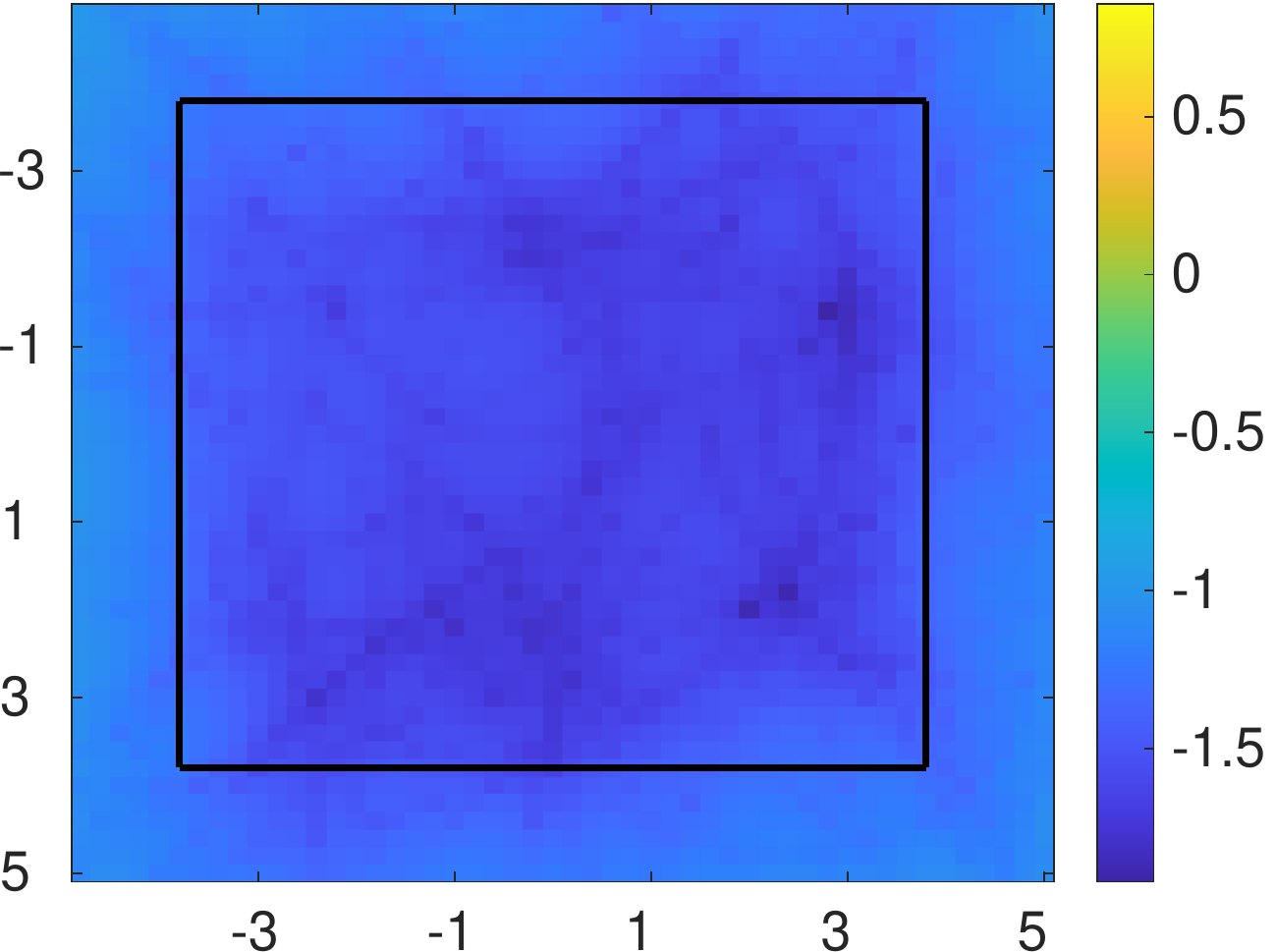}}%
\centering\begin{tabular}{@{}c@{\hskip 0.4cm}c@{\hskip 0.4cm}c@{\hskip 0.4cm}c@{}}
&\textbf{deviation} & \textbf{uncertainty} & \textbf{coverage}\\ 
\rowname{BLR}&
\includegraphics[width=.3\linewidth]{2D_rms_u_cv_heatmaps_rms_5_7_.pdf}&
\includegraphics[width=.3\linewidth]{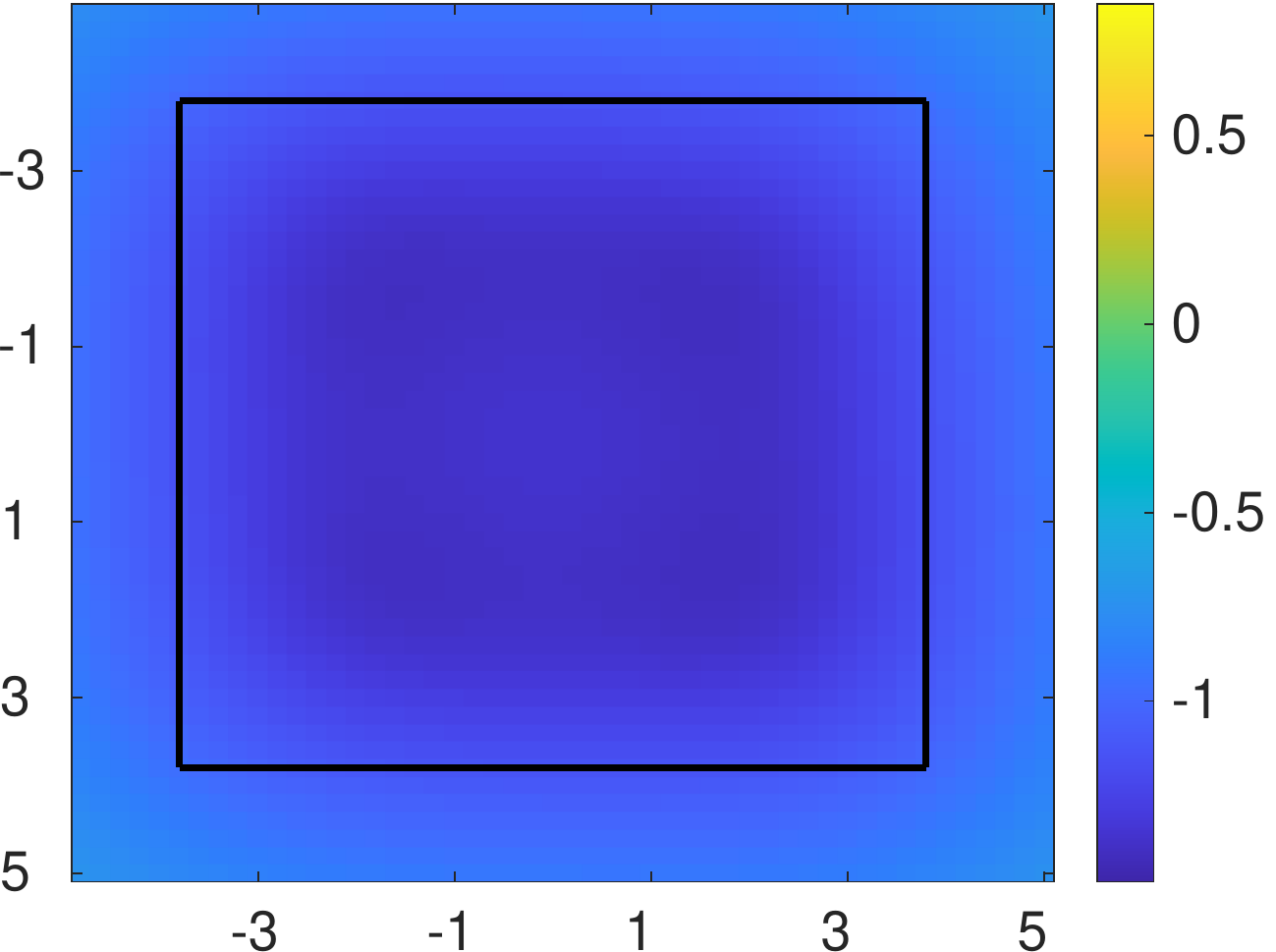}&
\includegraphics[width=.3\linewidth]{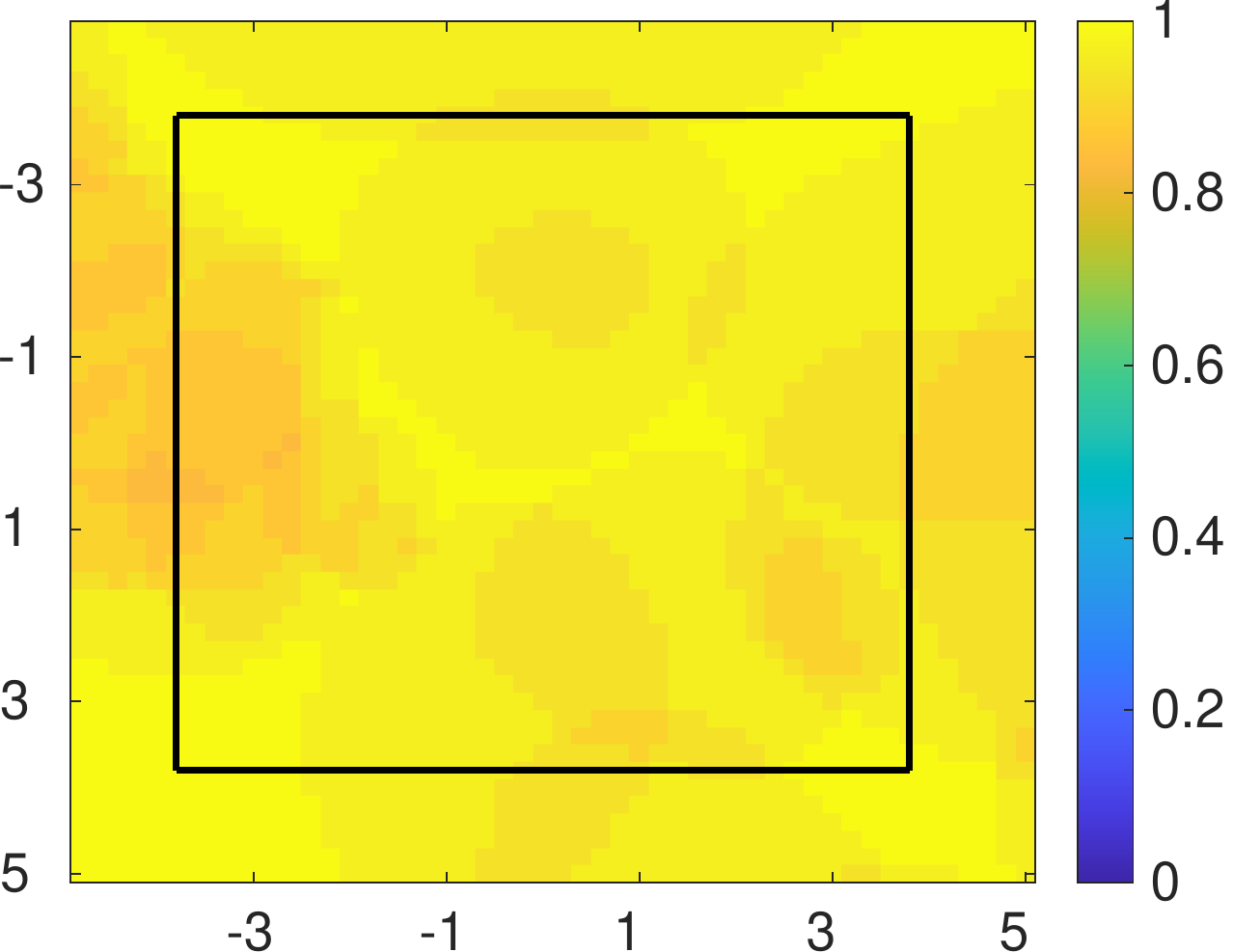}\\[1ex]
\rowname{BD}&
\includegraphics[width=.3\linewidth]{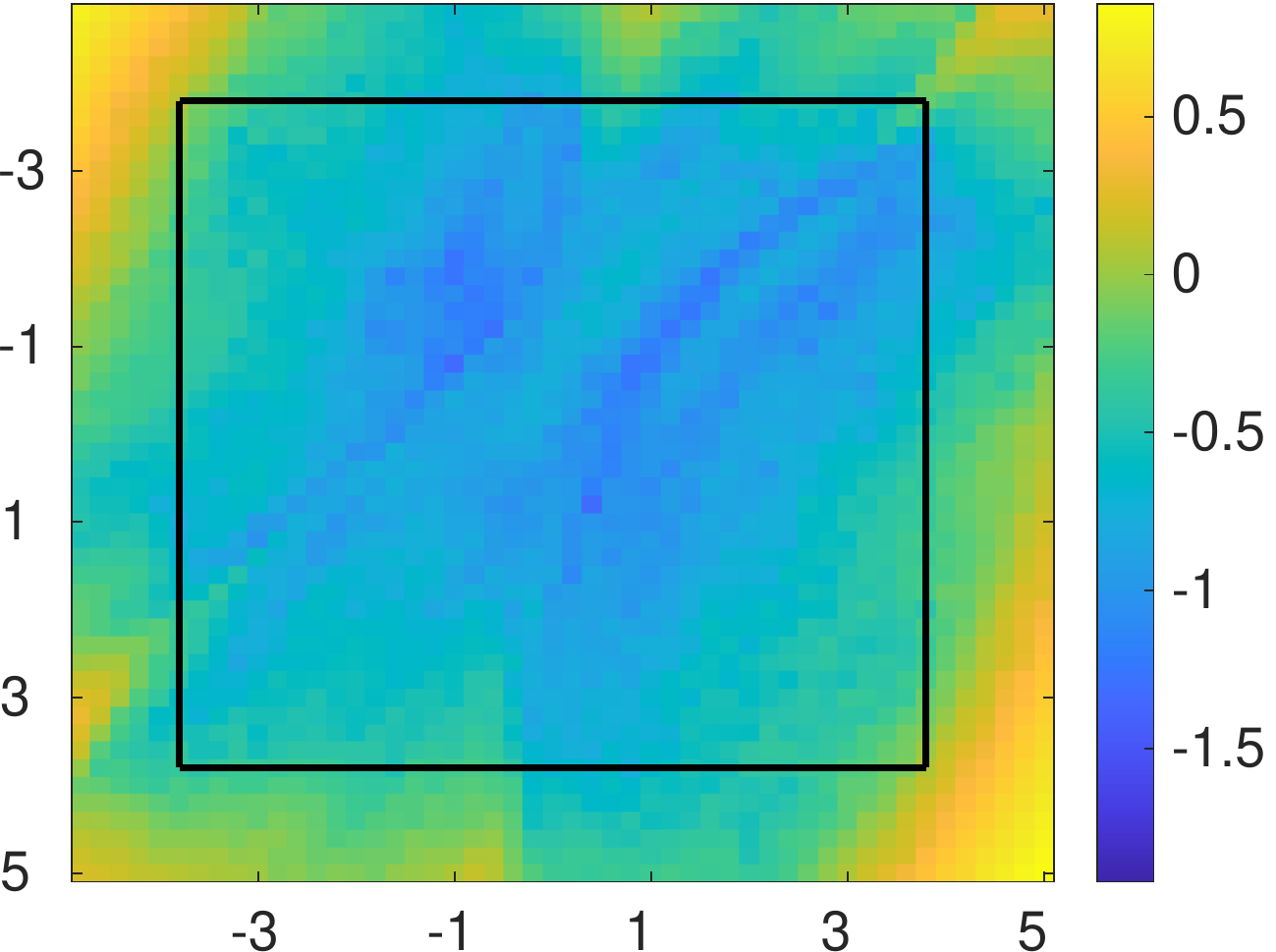}&
\includegraphics[width=.3\linewidth]{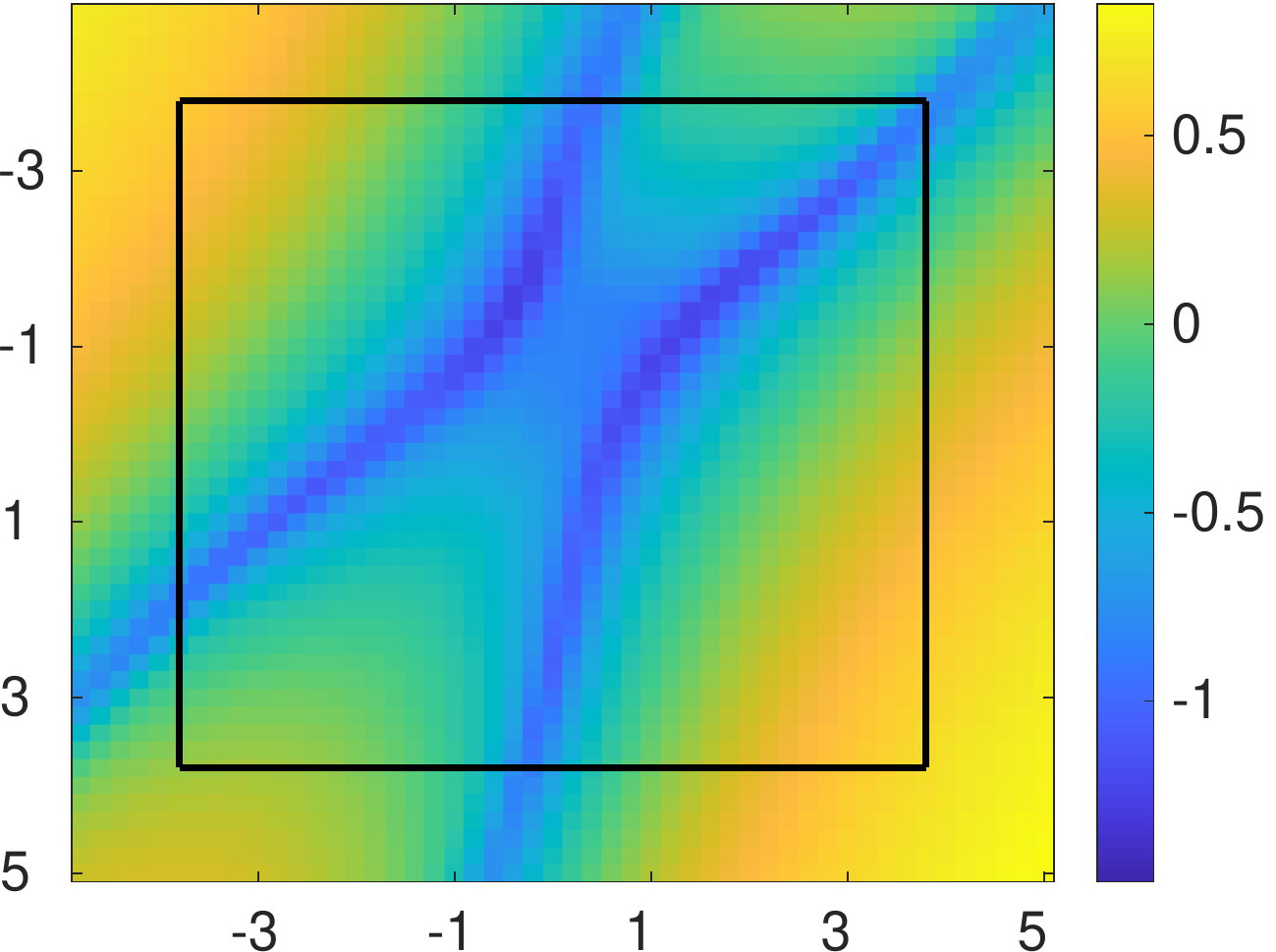}&
\includegraphics[width=.3\linewidth]{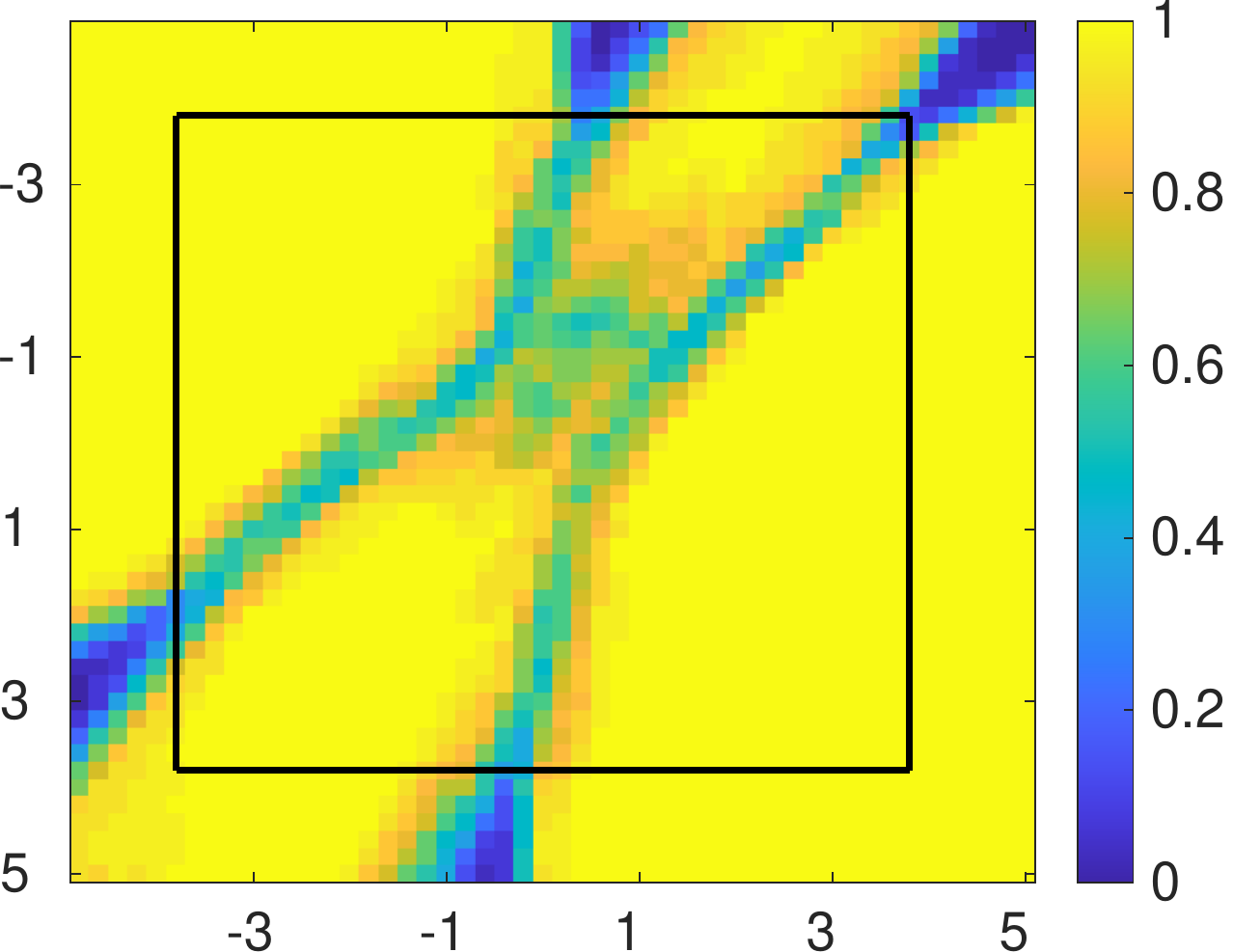}\\[1ex]
\rowname{VD}&
\includegraphics[width=.3\linewidth]{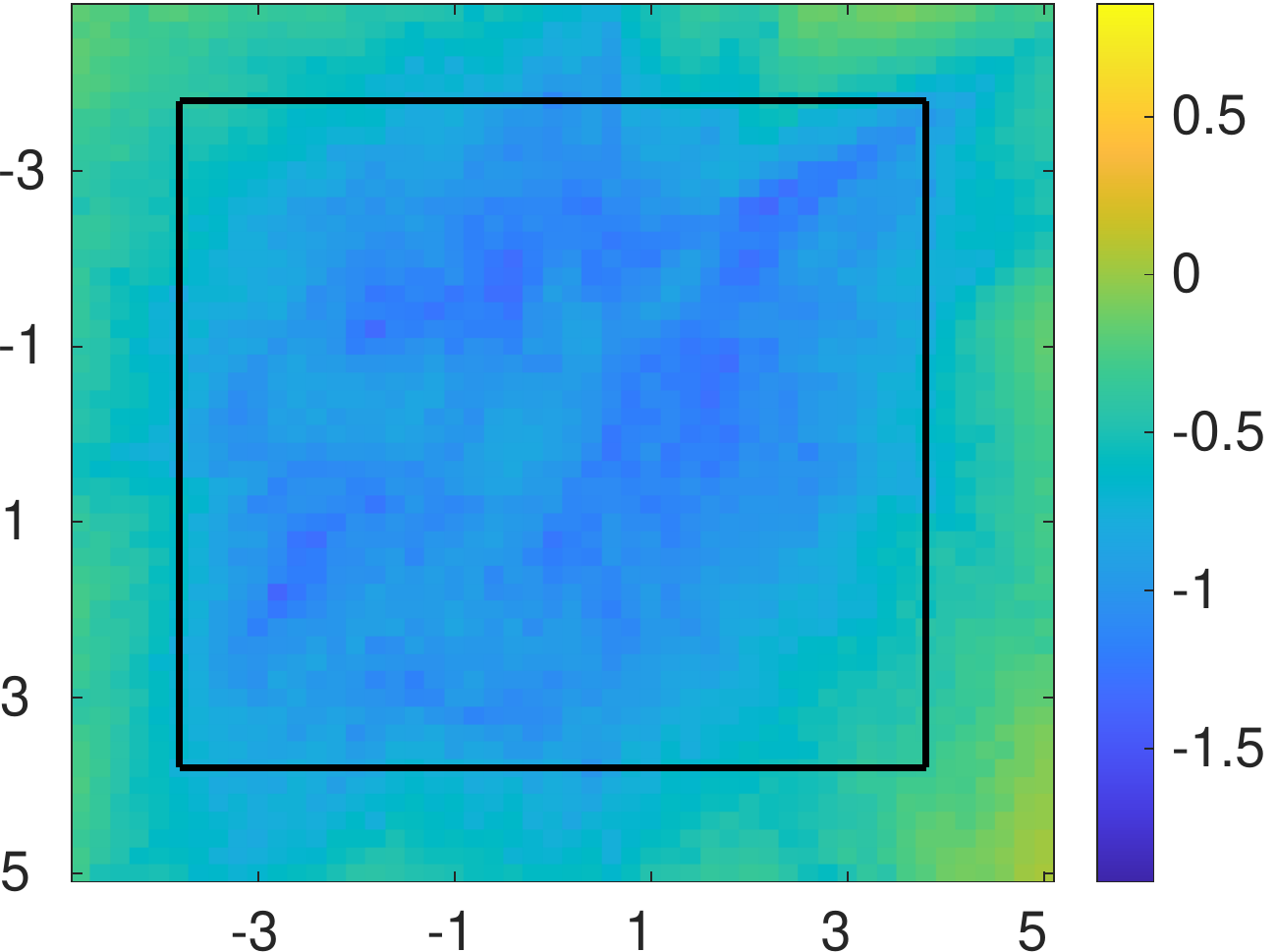}&
\includegraphics[width=.3\linewidth]{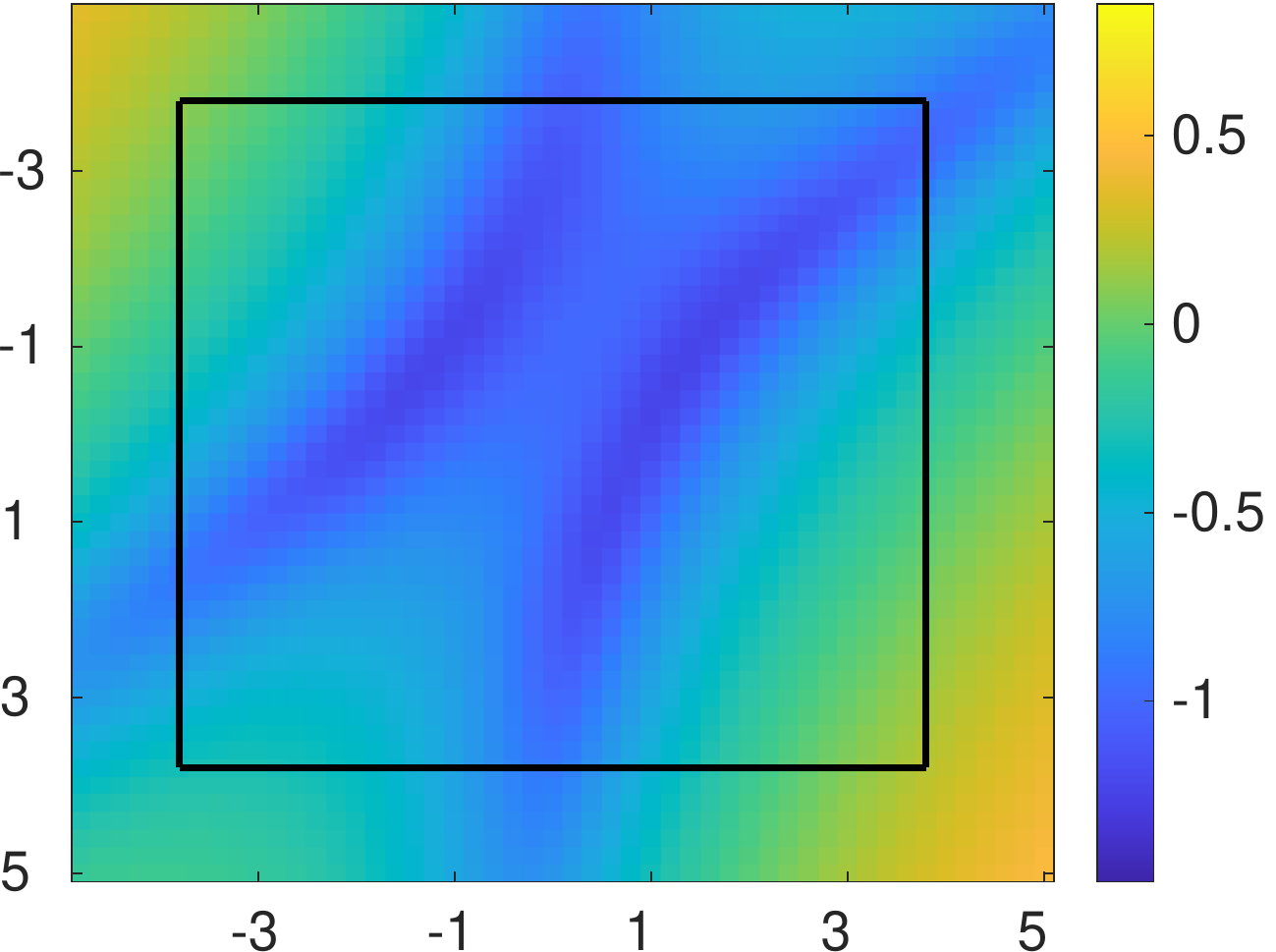}&
\includegraphics[width=.3\linewidth]{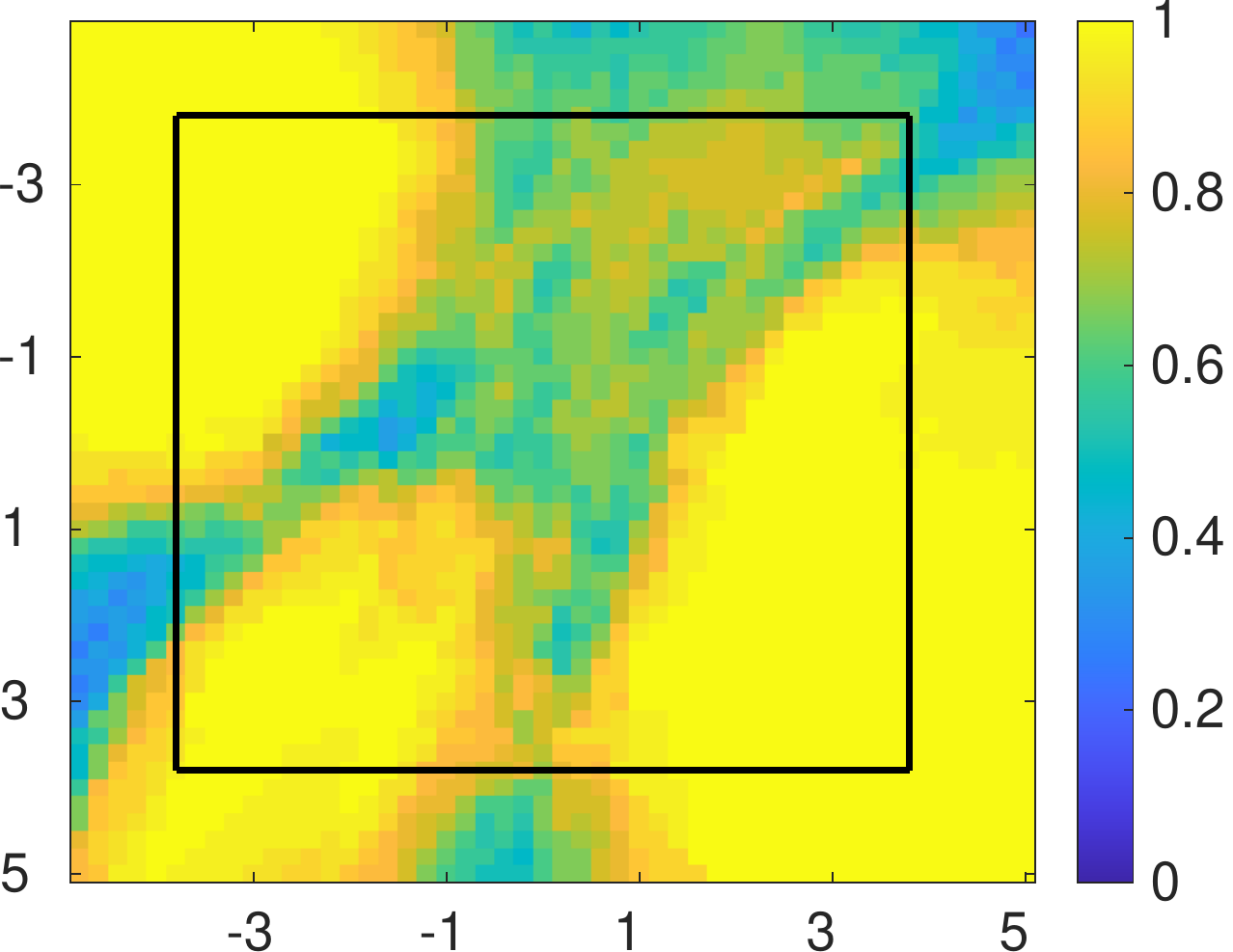}\\[1ex]
\rowname{EnsBS}&
\includegraphics[width=.3\linewidth]{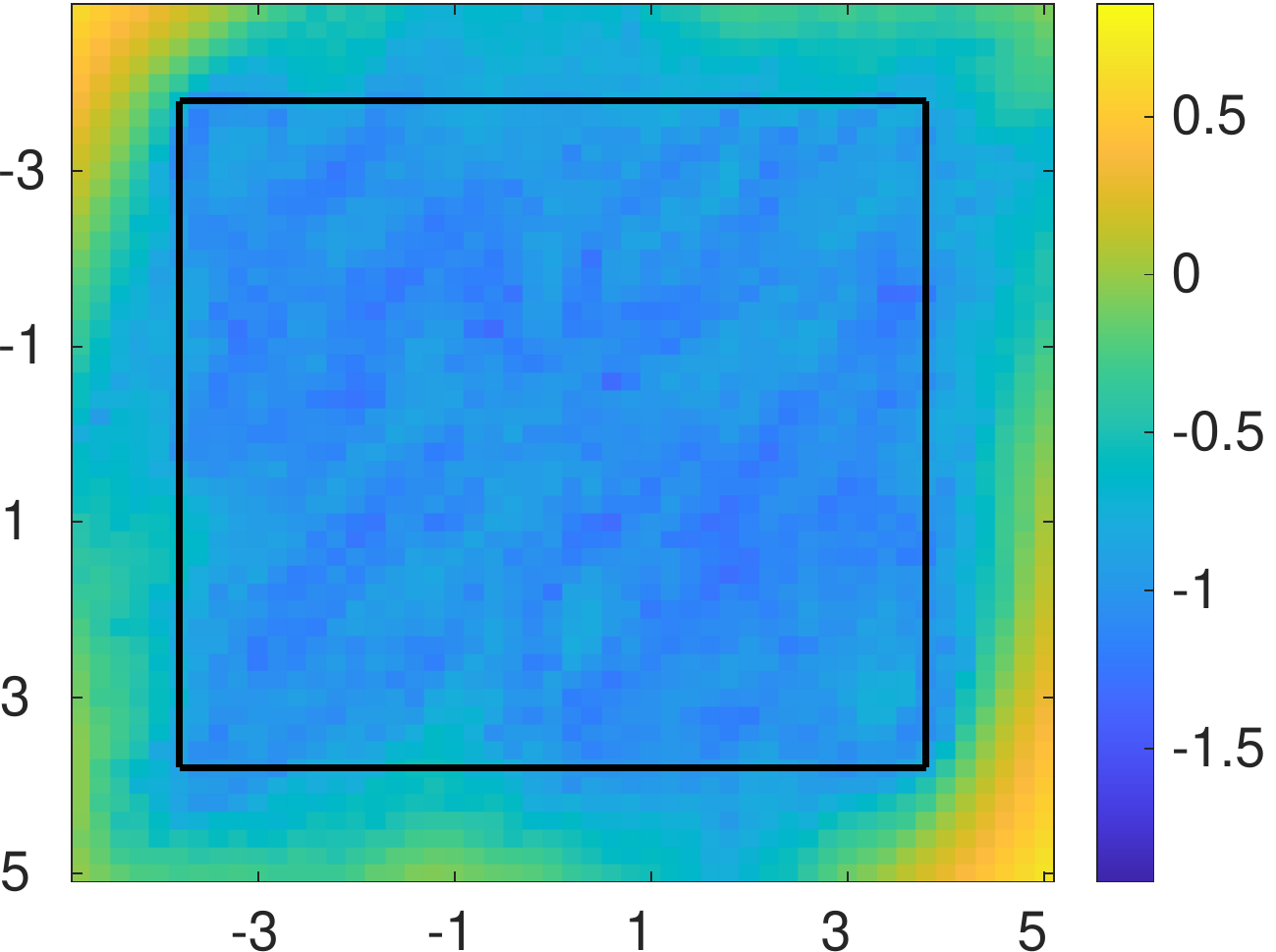}&
\includegraphics[width=.3\linewidth]{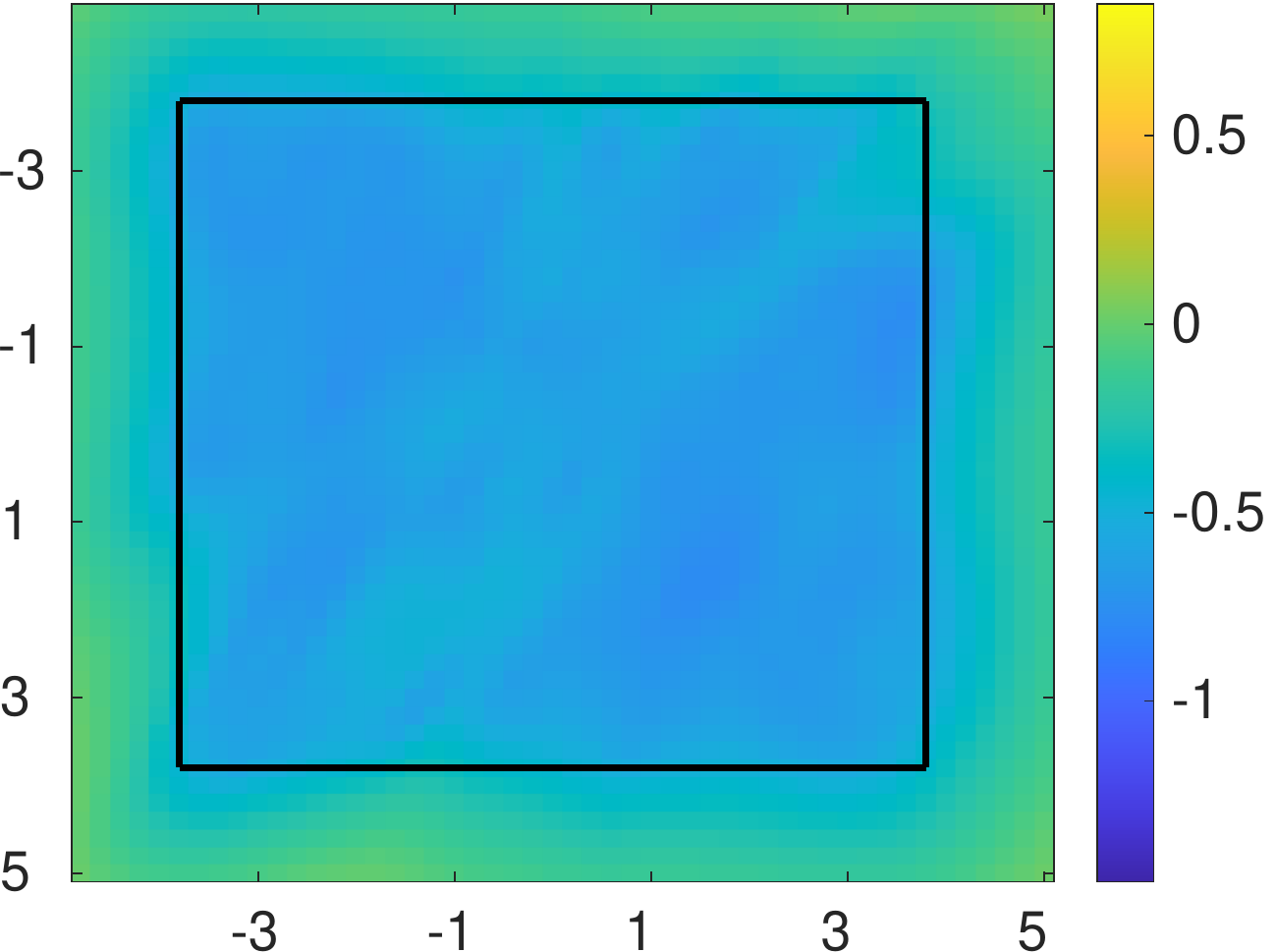}&
\includegraphics[width=.3\linewidth]{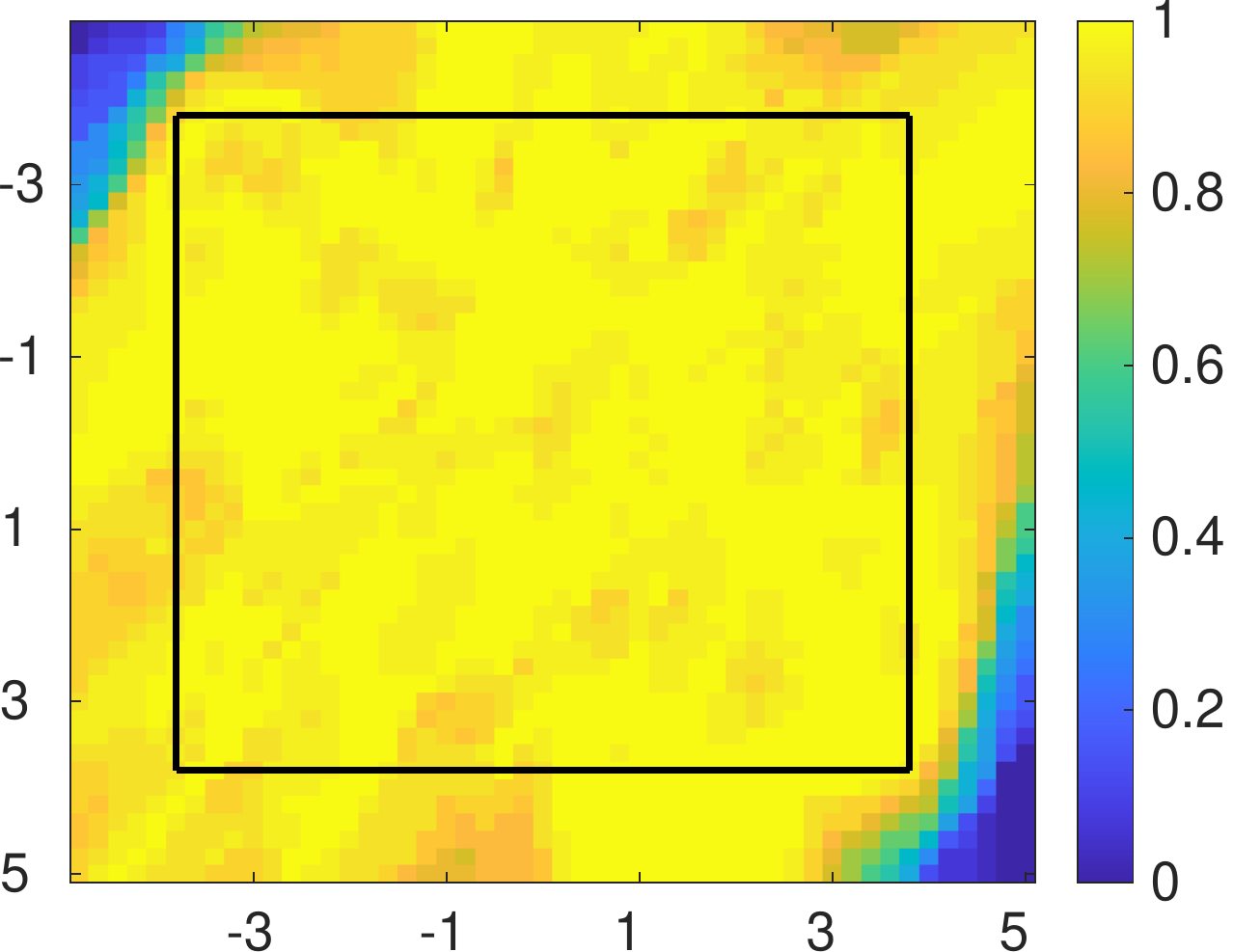}\\[1ex]
\end{tabular}
\caption{Results for the experiment \ref{E3} of the anchor model and different methods from deep learning for regular gridded test data. Every plot of the 9 visible plots arranged in a tabular, contains one heat map and shows the objectives introduced at the beginning of section \ref{sec:experiments} averaged over 30 repeated samplings: the deviation from the ground truth (first column, in $\log_{10}$ scale),  uncertainty (second column, in $\log_{10}$ scale) and  the coverage (third column). The different rows correspond to different methods with abbreviations are as in the beginning of section \ref{sec:results}.}
\label{fig:E3_heatmaps}
\end{figure}

\section{Conclusion and discussion} \label{sec:ConclusionOutlook}

In this article we described a method to create benchmarking test problems for uncertainty quantification in deep learning. We demonstrated how the linearity of a generating model in its parameters can be exploited to create an anchor model that helps to judge the quality of the prediction and the quantified uncertainty of deep learning approaches. Within this set-up it is possible to design generic data sets with explicit control over problem complexity and input dimensionality. For various test problems we studied the performance of several approaches for uncertainty quantification in deep regression. Common statistical metrics were used to compare the performance of these methods with each other and with the anchor model. In particular we assessed the behavior of these methods under increasing problem complexity and input dimensionality.
The flexibility to design test problems with specified complexity, along with the availability of a reference solution, qualifies the proposed framework to form the basis for designing set of benchmark problems for uncertainty quantification in deep regression.

\vskip 0.2in

\bibliography{bibfile_uq4dl}

\end{document}